\pdfoutput=1

\documentclass{article} %
\usepackage{iclr2027_conference,times}

\usepackage{amsmath,amsfonts,bm}

\def\eqref#1{equation~\ref{#1}}

\def\1{\bm{1}}

\def\vzero{{\bm{0}}}

\def\vmu{{\bm{\mu}}}
\def\vtheta{{\bm{\theta}}}

\def\vu{{\bm{u}}}

\def\vx{{\bm{x}}}

\def\vz{{\bm{z}}}

\def\mI{{\bm{I}}}

\def\mS{{\bm{S}}}

\def\mSigma{{\bm{\Sigma}}}

\DeclareMathAlphabet{\mathsfit}{\encodingdefault}{\sfdefault}{m}{sl}
\SetMathAlphabet{\mathsfit}{bold}{\encodingdefault}{\sfdefault}{bx}{n}

\def\gB{{\mathcal{B}}}
\def\gC{{\mathcal{C}}}
\def\gD{{\mathcal{D}}}

\def\gF{{\mathcal{F}}}

\def\gK{{\mathcal{K}}}
\def\gL{{\mathcal{L}}}

\def\gN{{\mathcal{N}}}

\def\gY{{\mathcal{Y}}}

\def\sQ{{\mathbb{Q}}}

\newcommand{\E}{\mathbb{E}}

\newcommand{\R}{\mathbb{R}}

\newcommand{\Var}{\mathrm{Var}}

\newcommand{\Cov}{\mathrm{Cov}}

\usepackage{amssymb,amsthm}
\usepackage{thm-restate}
\usepackage{hyperref}
\usepackage{url}
\usepackage[inline]{enumitem}
\hypersetup{hidelinks}
\usepackage{graphicx}
\usepackage{float}
\usepackage{booktabs}
\usepackage{xcolor}

\newcommand{\parbold}[1]{\textbf{#1}}

\title{What ``Form'' of Uncertainty Does Conformal Prediction Quantify?}

\author{Soroush H.~Zargarbashi\textsuperscript{1}, Mohammad Sadegh Akhondzadeh\textsuperscript{2}, \\
\bfseries Rebekka Burkholz\textsuperscript{1}, Aleksandar Bojchevski\textsuperscript{2} \\[2pt]
{\normalfont\textsuperscript{1}CISPA Helmholtz Center for Information Security \quad \textsuperscript{2}University of Cologne}
}

\newcommand{\xtest}{\vx_{n+1}}
\newcommand{\dcal}{\gD_{n}}

\newcommand{\daugmented}{\gD_{n+1}^{(c)}}

\providecommand{\vpsi}{{\bm{\psi}}}
\providecommand{\mPsi}{{\bm{\Psi}}}

\newtheorem{theorem}{Theorem}[section]
\newtheorem{lemma}[theorem]{Lemma}

\theoremstyle{definition}

\theoremstyle{remark}

\iclrfinalcopy
\begin{document}

\maketitle
\lhead{Preprint. Under review.}

\begin{abstract}
   Conformal prediction (CP) is a post-hoc calibration method that returns prediction sets with distribution-free coverage guarantees, and the size of the output set is often interpreted as a notion of uncertainty. We show that the difference in set size between full and split CP is directly connected to Bayesian posterior variance, which is the epistemic uncertainty of the model. Essentially, this difference in set size is an aggregation of two opposite effects: (1) better model accuracy and therefore smaller set size achieved by training over extra calibration points, and (2) the ``formability'' of the model at test point towards every possible label. The latter effect vanishes with stronger models or more calibration data (as epistemic uncertainty is known for that behavior). It turns out that the model’s formability itself (without the quantization effect of prediction sets, and through estimating the posterior variance) is a useful tool to quantify the epistemic uncertainty. In our empirical evaluation it performs on par with the state of the art in many setups.
\end{abstract}

\begin{figure}[b]
    \centering
    \includegraphics{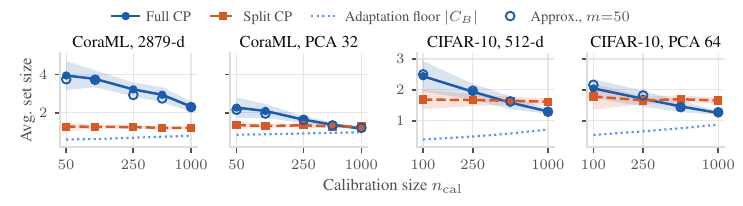}
    \caption{Comparison of full (blue, solid) vs.\ split (orange, dashed) CP, against the adaptation floor $|\gC_B|$ (see {\autoref{sec:formability}}), over $n_{\mathrm{cal}}$ and feature dimension $K$, and for [Left] CoraML (transductive node-classification, GCN), and [Right] CIFAR-10. The approximation is with $m=50$ samples (see \autoref{sec:experiments}).}
    \label{fig:full-vs-split-calib-pca}
\end{figure}

\section{Introduction}
A model's prediction can be uncertain, either since the data is hard to predict (aleatoric), or because the model is not confident in its prediction due to lack of information (epistemic). Conformal prediction (CP) is a widely used approach in risk-critical applications where quantifying total predictive uncertainty is crucial. Given any model $f(\cdot;\vtheta)$, CP returns a prediction set $\gC(\xtest)$ for a test point $\xtest$ that is guaranteed to include the true label with user-specified probability $1-\alpha$. This guarantee is distribution-free; namely, it holds regardless of the underlying data distribution, or the model family \citep{angelopoulos2024theoretical}. CP requires an exchangeable set of calibration points, and can be applied in two ways: \begin{enumerate*}[label=(\roman*)]
    \item fine-tuning the model jointly on the calibration set and the test point for every possible pseudo-label (full CP), or
    \item only evaluating the model on the calibration points (split CP).
\end{enumerate*} 
Full CP is known to be more data-efficient, which is particularly appealing in cases with label scarcity \citep{lei2018distribution} (e.g. semi-supervised graph node classification). The argument is that in these settings the data is not divided into separate training and calibration sets, which potentially results in better model performance, and smaller sets. 

In contrast, \citet{zong2025comparative} show that full CP doesn't always outperform split CP without specifying a definitive reason, leaving this question open: \emph{when and by how much full CP can be more efficient, and what can be concluded from their performance gap?} It turns out (as shown in \autoref{fig:full-vs-split-calib-pca}, and \autoref{fig:train-calib-grid}), the performance gap (set size difference) {reduces with} a larger dataset (calibration, or training), {and a lower feature dimension}. Intuitively, while in most cases a larger dataset (combination of training and calibration) improves the model's performance, fine-tuning on the test point with a pseudo-label deforms the model towards that label. The strength of this deformation is governed by the calibration density, and the confidence from {pretraining} both as resisting factors. Strong deformation includes wrong pseudo-labels in the set, resulting in worse overall performance.

In particular, we define \emph{formability} at a test point, as the difference between the predictive probability {of any label} between two models: one {fine-tuned on a} calibration set, and the other on the union of the calibration set and the pseudo-labeled test point (on that label). Under Bayesian training, where the epistemic uncertainty is quantified as the disagreement among the plausible (under the posterior) parameters, we show that formability is directly related to the epistemic uncertainty of the predictive distribution at the test point; namely we bound the posterior variance, by a function of formability. This also shows the reason why by increasing calibration set size the difference between full and split CP disappears: the formability (epistemic uncertainty) is known to reduce with more data. 

\section{Background} 
\label{sec:background}
Let $f(\cdot;\vtheta)$ be a pretrained model from the family $\gF$ specified by the parameters $\vtheta$. We use $p_\vtheta(y\mid \vx)$ for the predictive label distribution of $f$ at the input $\vx$. Let $\dcal=\{(\vx_i,y_i)\}_{i=1}^n$ be a holdout calibration set jointly exchangeable with the test point $(\vx_{n+1},y_{n+1})$.

\parbold{Split CP \citep{vovk2005algorithmic} (SCP).} With $r^\vtheta_i=p_\vtheta(y_i\mid \vx_i)$, define the threshold 
\begin{align}
    q_\vtheta = Q_{\alpha,n}(r_{1:n})
    :=\sQ_{\alpha}\left(
    \sum_{i=1}^{n}\frac{1}{n+1}\delta_{r^\vtheta_i}
    +\frac{1}{n+1}\delta_{\infty}
    \right).
    \label{eq:lower-conformal-quantile}
\end{align}
 and the prediction set $\gC_\vtheta^{\mathrm S}(\vx_{n+1})=\{c\in\gY:p_\vtheta(c\mid \vx_{n+1})\geq q_\vtheta\}$. This set has a distribution-free coverage guarantee $\Pr[y_{n+1}\in\gC_\vtheta^{\mathrm S}(\vx_{n+1})]\geq 1-\alpha$. Other scores are discussed in \autoref{app:train-vs-calib-aps}. 

\parbold{Full CP.}
{We define} the augmented dataset $\daugmented=\dcal\cup\{(\vx_{n+1},c)\}$,  for a pseudo-label $c$. Full conformal prediction \citep{vovk2005algorithmic} (FCP) defines the prediction set 
$
    \gC_\vtheta^{\mathrm F}(\vx_{n+1})=\{c\in\gY:p_\vtheta^{(c)}(c\mid \vx_{n+1})\geq Q_{\alpha,n}(r_{1:n}^{(c)})\}$ 
for $r_i^{(c)}=p_\vtheta^{(c)}(y_i\mid \vx_i)$, where $p_\vtheta^{(c)}$ is the predictive distribution of the model after fine-tuning on $\daugmented$\footnote{Full CP can be applied by both training an initial model, or fine-tuning a pre-trained model.}. This prediction set also comes with the same distribution-free guarantee. In both cases, the guarantee comes from the fact that the true scores $p_\cdot(x_i, y_i)$ are exchangeable for $i \in [n+1]$. The difference is that FCP fine-tunes a single model per test point, per label. The potential use case of this computationally expensive approach is when the labeled dataset is small, and the model can benefit from training on the entire dataset (training and calibration).

\parbold{Bayesian deep learning.} Instead of a point estimate $\vtheta$ (deterministic value for parameters), Bayesian deep learning \citep{neal1996bayesian,wilson2020bayesian} {defines a} distribution $\pi(\vtheta)$ over the model parameters. 
As a result, the predictive distribution is also defined as $p_\pi(y\mid\vx):=\E_{\vtheta\sim\pi}[p_\vtheta(y\mid \vx)]$ instead of a single $p(y\mid \vx;\vtheta)$.
Training a Bayesian model is equivalent to finding the posterior parameter distribution $\pi(\vtheta\mid\gD)$ given the dataset $\gD$ as evidence, and a prior $\pi_0(\vtheta)$ over the parameters (e.g. $\gN(\boldsymbol{0}, \sigma \mI)$ for initialization).
Under this definition, the epistemic uncertainty is quantified as the disagreement among parameters that are plausible under the posterior -- measured as the variance $\Var_{\vtheta\sim\pi}[p_\vtheta(\cdot\mid \vx)]$ \citep{pmlr-v244-sale24a}. Notably, as \citet{jimenez2025machine} argue, one cannot quantify the aleatoric uncertainty unless the epistemic uncertainty is sufficiently low. 

\parbold{Basin decomposition.} The support of a probability distribution (e.g. $\pi(\vtheta)$) can be decomposed into disjoint basins of attraction $\gB_1,\ldots,\gB_K$  with $\cup_{k=1}^K\gB_k=\mathrm{Support}(\pi)$, and $\gB_k\cap\gB_{k'}=\emptyset$ for $k\neq k'$. Each basin is a connected region of the parameter space that contains a local optimum $\hat\vtheta_k$ of the loss function. In \autoref{sec:map-refits}, we discuss that the posterior predictive distribution can be decomposed into a mixture of distributions over each basin{. A} better approximation of the variance considers basins other than only the one containing the MAP as well.

\begin{figure}[t]
  \centering
  \includegraphics{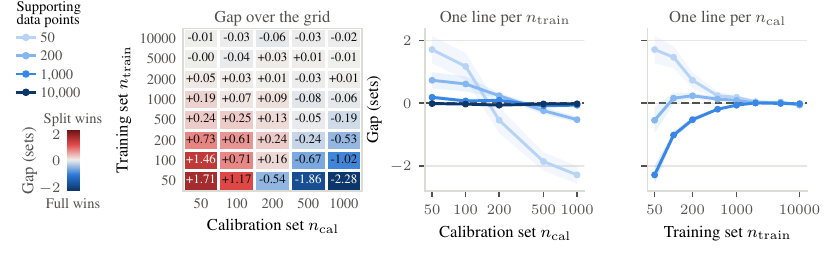}
  \caption{
    Performance gap between $|C_{\mathrm{full}}|$ and $|C_{\mathrm{split}}|$ over a grid of training and calibration set sizes. This is an {extension} of the result shown in \autoref{fig:full-vs-split-calib-pca}. 
    }
  \label{fig:train-calib-grid}
\end{figure}

\section{Formability}
\label{sec:formability}

\subsection{Motivating Experiment}
\label{sec:motivating-experiment}

First we evaluate the {claim that} ``in label scarcity, full CP outperforms split CP'' by comparing full and split conformal prediction on two simple setups with images (\texttt{CIFAR-10}{; \citealp{krizhevsky2009learning}}), and graphs (\texttt{CoraML}{; \citealp{bojchevski2018deep}}). It is already mentioned {in a couple} of works that in some cases full CP is not better than split CP in set size {\citep[e.g.,][]{barber2023conformal, zong2025comparative}}. {As shown in} \autoref{fig:full-vs-split-calib-pca}, {however, by} increasing the calibration set size, or reducing the feature dimension {(using PCA), the} performance gap is reduced, until the two methods cross over. These two observations suggest that the performance gap is not governed by the amount of data alone, but by how strongly the model is already determined around the test point. Notably, the argument of ``full CP leverages a better model'' still {serves} as a lower bound on the set size, which full CP converges to, by increasing the size of the calibration set. 

\begin{figure}
    \includegraphics{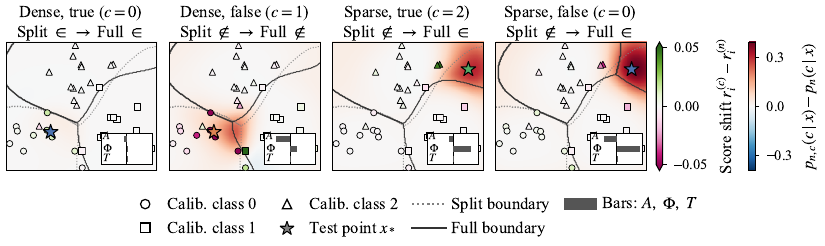}
    \caption{The formability $p_{n,c}(c\mid\vx)-p_n(c\mid\vx)$ by fine-tuning on the pseudo-label $c$ at the test point (star), in a dense [left pair], and a sparse [right pair] (calibration) neighborhood. {Pairs are a} correct, and a wrong label $c$. Calibration points are colored by the shift of their true score ({negligible at many points}), and the bars report $A_{n,c}$, $\Phi_{n,c}$, and $T_{n,c}$ from \autoref{eq:margin-decomposition}.}
    \label{fig:score-anatomy}
\end{figure}

An intuitive explanation for such behaviour is that when the training and calibration data is {scarce}, the model {can still be} deformed towards any label at the test point. Therefore, it is very easy for every (true or false) label of $\xtest$ to pull its own score above the threshold. 
To observe this mechanism in isolation, we show a synthesized setup where every score is visible on a 2-dim landscape (\autoref{fig:score-anatomy}, details are in \autoref{app:score-anatomy}). {We} compare a test point inside a dense calibration neighborhood to one in a sparse support. In all cases, fine-tuning on the pseudo-label deforms the predictive distribution towards the {pseudo-label}. {However, for} the one in the dense region, the calibration density resists this deformation, and the self-pull remains small. In the sparse region there are no proximate data to resist. Hence, every pseudo-label pulls its own score {above the threshold easily}. As a result, the wrong label enters the full set exactly where the model was free to deform. Therefore the performance gap in CPs is governed by the significance of the new test data point w.r.t. the model: how much the model can deform at that point (which is {inversely related to} the size of $\dcal$, and the strength of the pretrained model as a resisting factor). Namely, this deformation is connected to the epistemic uncertainty -- a model with more knowledge (training, and calibration data) {is harder to form} towards wrong labels. 

This direct connection between formability and epistemic uncertainty is particularly important when it compares with methods like softmax entropy. Naturally, the softmax entropy, when evaluated on a single predictor can not distinguish between the two types of uncertainty. \autoref{fig:epistemic-sweeps} shows that by replicating the setup from \autoref{fig:score-anatomy} while moving the generative Gaussians close to each other. At the center, the epistemic uncertainty decrease, while aleatoric uncertainty increase. On a single label distribution however, it is presented as $[1/k]_k$ the entire time. Therefore the softmax entropy does not react to this change, while the formability variance does. 

\subsection{Theory}
\label{sec:theory}

\begin{figure}
    \centering
    \includegraphics{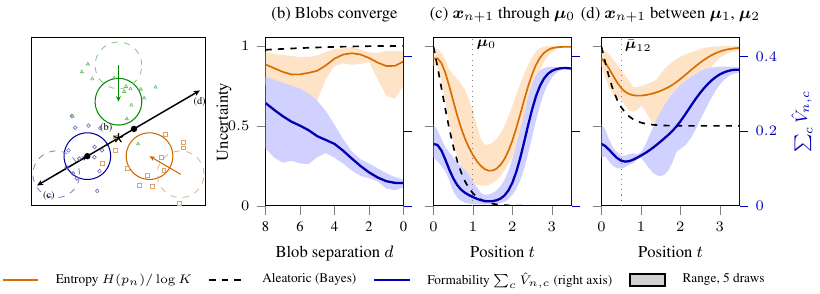}
    \caption{Softmax entropy vs formability variance. {The setup is} similar to \autoref{fig:score-anatomy} (see \autoref{app:epistemic-sweeps}). [Left] Data, and direction of movements. Circles show the mean and standard deviation (radius). We move $\xtest$ from the center to between two blobs (alea $\uparrow$, and epis $\downarrow$), and from the other side towards the center of the other blob (both $\downarrow$). [b] Uncertainty at center (star) when blobs move closer (epis $\downarrow$, and alea $\uparrow$), \emph{only formability reacts to it}. [c], and [d] show it when moving towards the center of one blob (both $\downarrow$), and between two blobs (epis $\downarrow$, alea $\uparrow$). }
    \label{fig:epistemic-sweeps}
\end{figure}

Let $m_\circ(\vx, y):=p_\circ(y\mid \vx)-q_\circ$ be the margin of the label $y$ at $\vx$ in {any} of these predictive regimes: \begin{enumerate*}[label=(\roman*)]
    \item split CP with the pretrained predictive distribution $p_\vtheta$ (call it $m_0$),
    \item full CP fine-tuned on the augmented set $\daugmented$ with predictive distribution $p_{n,c}$ (call it $m_{n,c}$), and 
    \item CP fine-tuned only on the calibration set $\dcal$ with predictive distribution $p_n$ (call it $m_n$).
\end{enumerate*} Note that (iii) is not a valid CP, as the calibration and test points are not exchangeable, but it {represents} the effect of ``better model from more data''. In all cases, a label is in the prediction set if its margin is non-negative. We decompose the full margin and define three new variables $A_{n,c}$, $\Phi_{n,c}$, and $T_{n,c}$ as
\begin{align}
    m_{n,c}(\vx,c) = m_0(\vx,c) &+ (m_n(\vx,c) - m_0(\vx,c)) + (m_{n,c}(\vx,c) - m_n(\vx,c))
    \label{eq:margin-decomposition}\\
    =m_0(\vx, c) &+\underbrace{\bigl[p_\vtheta^{(n)}(c\mid \xtest)-p_\vtheta(c\mid \xtest)\bigr]-\bigl(q_\vtheta^{(n)}-q_\vtheta\bigr)}_{A_{n,c}}\nonumber\\
    &+\underbrace{p_\vtheta^{(c)}(c\mid \xtest)-p_\vtheta^{(n)}(c\mid \xtest)}_{\Phi_{n,c}}
    -\underbrace{\bigl(q_\vtheta^{(c)}-q_\vtheta^{(n)}\bigr)}_{T_{n,c}}. \nonumber
\end{align}
The above decomposition is trivial, but each of the introduced terms has an interpretation: $A_{n, c}$ captures the adaptation-driven efficiency which is the result of {fine-tuning on the calibration set}. It simulates the effect from which we expected full CP to be more efficient. $\Phi_{n,c}$, and $T_{n,c}$ are the effects of training on the additional pseudo-labeled test point (we call the total $\bar\Phi_{n,c}:=\Phi_{n,c}-T_{n,c}$ the net formability). Adding $(\xtest, c)$ to the training set pulls the model towards that label (increasing the margin), but at the same time, it affects the other calibration points, which can change the quantile threshold. Ultimately, the full and split CP {set size is} only a quantization of the margins; i.e. 
\begin{align*}
| \gC^F(\xtest) |-|\gC^S(\xtest)| 
                                  &=\sum_{c\in\gY}\Bigl[
    \1\bigl\{m_0+A_{n,c}+\Phi_{n,c}-T_{n,c}\geq0\bigr\}
    -\1\bigl\{m_0\geq0\bigr\}
    \Bigr].
\end{align*}

\parbold{Under Bayesian Learning.}
Discussing this decomposition under Bayesian learning has two advantages: 
\begin{enumerate*}[label=(\roman*)]
    \item Bayesian learning comes with a well-defined notion of epistemic uncertainty; i.e. the variance of the posterior predictive distribution, and
    \item in the Bayesian framework, we can write fine-tunings {as a sequence} of Bayesian updates, which decouples the effect of the additional pseudo-label from the calibration data.
\end{enumerate*}
Let $\pi_0(\vtheta)$ be the prior parameter distribution centered around $\vtheta_0$ (initial or pretrained model), $\pi_n(\vtheta)$ be the posterior distribution after observing the calibration set $\dcal$ (we call it ``calibration posterior''), and $\pi_{n,c}(\vtheta)$ be the posterior distribution with $\daugmented = \dcal \cup (\xtest, c)$ as evidence (we call it ``augmented posterior''). 
By sequential Bayesian updating (\autoref{lem:sequential-updating}), the posterior distribution {$\pi_{n,c}$ can} be computed using {Bayes' rule} with $\{(\vx_{n+1},c)\}$ as {a single evidence}, and $\pi_n$ as the prior.
{For ease of notation}, let
\begin{align}
    w_{n,c}:=p_{\pi_n}(c\mid \vx_{n+1})
    =\E_{\vtheta\sim\pi_n}[p_\vtheta(c\mid \vx_{n+1})],
    \label{eq:candidate-mass}
\end{align}
and assume $w_{n,c}>0$.
In Bayesian learning, we define the formability on $(\vx_{n+1}, c)$ as
\begin{align}
    \Phi_{n,c}(\vx_{n+1})=\E_{\vtheta\sim\pi_{n,c}}[p_\vtheta(c\mid \vx_{n+1})]-\E_{\vtheta\sim\pi_n}[p_\vtheta(c\mid \vx_{n+1})].
    \label{eq:formability} 
\end{align}

The following lemma shows that the formability is directly related to the epistemic uncertainty of the predictive distribution at $\vx_{n+1}$.

\begin{restatable}{lemma}{formabilityvariancelemma}
\label{lem:formability-variance}
The (per-label) variance of the calibration posterior predictive distribution at $\vx_{n+1}$ is proportional to the formability at $\vx_{n+1}$:
\begin{align}
    \Var_{\vtheta\sim\pi_n}[p_\vtheta(c\mid \vx_{n+1})]
    =w_{n,c}\Phi_{n,c}(\vx_{n+1}), \quad \text{and} 
    \label{eq:variance-from-formability}
\end{align}
{therefore,} $0\leq \Var_{\vtheta\sim\pi_n}[p_\vtheta(c\mid \vx_{n+1})] \leq \Phi_{n,c}(\vx_{n+1})\bigl(1-\Phi_{n,c}(\vx_{n+1})\bigr).$

As a result, for a finite label set, the total posterior variance (epistemic uncertainty) satisfies
\begin{align}
    \E_{\vtheta\sim\pi_n}\!\left[
        \left\|p_\vtheta(\cdot\mid \vx_{n+1})
        -p_{\pi_n}(\cdot\mid \vx_{n+1})\right\|_2^2
    \right]\nonumber
    &=\sum_{c\in\gY}w_{n,c}\Phi_{n,c}(\vx_{n+1})\\
    &\leq\sum_{c\in\gY}\Phi_{n,c}(\vx_{n+1})
    \bigl(1-\Phi_{n,c}(\vx_{n+1})\bigr).
    \label{eq:total-variance-from-formability}
\end{align}
\end{restatable}

The two terms $\Phi_{n, c}$, and $T_{n, c}$ in \autoref{eq:margin-decomposition} are governed by the same posterior on the pseudo-label. As both variance, and $w_{n, c}$ are positive, the formability $\Phi_{n,c}$ is positive. {This means that} training on the pseudo-label necessarily increases its score. At the same time, it also moves every calibration score, and therefore the threshold. To study {all of the above}, we define the \emph{forming kernel} as the covariance kernel over the posterior distribution:
\begin{align}
    \gK_n(z,z'):=\Cov_{\vtheta\sim\pi_n}\bigl(p_\vtheta(y\mid \vx),\,p_\vtheta(y'\mid \vx')\bigr),
    \qquad z=(\vx,y),\ z'=(\vx',y'),
    \label{eq:formability-kernel}
\end{align}
and use the shorthands $C_{i,c}:=\gK_n\bigl((\vx_i,y_i),(\vx_{n+1},c)\bigr)$ for the kernel between calibration points, and the pseudo-label hypothesis, $V_i:=\gK_n\bigl((\vx_i,y_i),(\vx_{i},y_i)\bigr) = \Var_{\vtheta\sim\pi_n}[p_\vtheta(y_i\mid \vx_i)]$, and the calibration-posterior variance at {the test point,} $V_{n,c}:=\Var_{\vtheta\sim\pi_n}[p_\vtheta(c\mid \vx_{n+1})] = \Phi_{n,c}(\vx_{n+1}) w_{n, c}$ (\autoref{lem:formability-variance}).

\begin{figure}[t]
    \centering
    \includegraphics{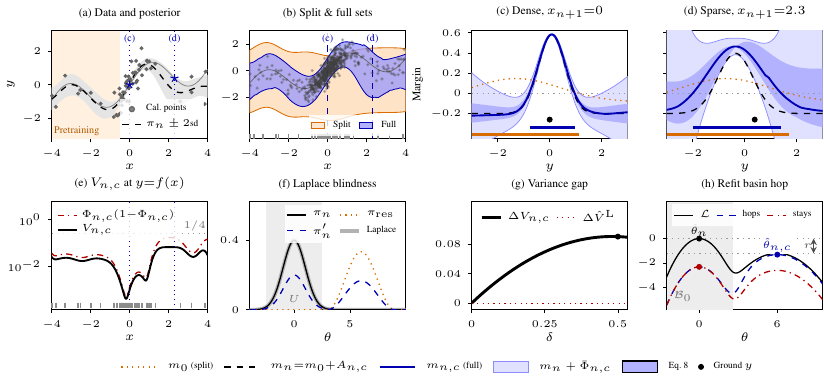}
    \caption{[a - e] \autoref{thm:full-split-decomposition} for Bayesian linear regression: [a] data, training covariate, and calibration points; [b] split and full sets along $x$; [c,d]  the margins $m_\circ$ from \autoref{thm:full-split-decomposition} in dense and sparse area; [e] posterior variance, and its {upper bound} from \autoref{eq:total-variance-from-formability}. Details are in \autoref{app:blr-check}. [f-h] {Illustration} of \autoref{prop:laplace-blind}, and \autoref{prop:map-candidate-selection}: [f] same Laplace approximation on two different distributions: $\pi_n$ (solid), and a second distribution $\pi_n'$ {(dashed); [g]} the variance difference by increasing the impact of $\pi_\mathrm{res}$. [h] The refit can jump to a different basin, i.e. $\pi_{n,c}$ can be outside the basin of the original MAP $\hat\theta_n$.
    }
    \label{fig:section3}
\end{figure}

\begin{restatable}{theorem}{fullsplitdecomposition}
\label{thm:full-split-decomposition}
Under Bayesian training, with scores $r_i^{(n)}=p_{\pi_n}(y_i\mid \vx_i)$, $r_i^{(c)}=p_{\pi_{n,c}}(y_i\mid \vx_i)$ and $q_\vtheta^{(n)}=Q_{\alpha,n}(r_{1:n}^{(n)})$, $q_\vtheta^{(c)}=Q_{\alpha,n}(r_{1:n}^{(c)})$, the following statements hold:
    (i) the update in each calibration score $r_i$ is exactly $r_i^{(c)}-r_i^{(n)}=\frac{C_{i,c}}{w_{n,c}}$ and bounded in absolute value by $\frac{\sqrt{V_i\,V_{n,c}}}{w_{n,c}}$.
    As a result, with $[u\pm v]$ denoting the closed interval $[u-v,\,u+v]$, we have (ii - Decomposition):
    \begin{align}
        m_{n,c} &\in
        m_0 + A_{n,c} + 
        \left[\frac{V_{n, c} - \max_i C_{i, c}}{w_{n,c}}, \frac{V_{n, c} - \min_i C_{i, c}}{w_{n,c}}\right]
        \\&\in
        \underbrace{m_0 + A_{n,c}}_{\text{defined as } m_n} + 
        \underbrace{\left[\frac{\sqrt{V_{n, c}}(\sqrt{V_{n, c}} \pm \max_i \sqrt{V_i})}{w_{n,c}}\right]}_{\substack{\text{net formability (epistemic) }}\bar\Phi_{n, c}}.
        \label{eq:bayesian-margin-decomposition}
    \end{align}
\end{restatable}

\autoref{thm:full-split-decomposition} shows that under Bayesian training, the set size for full CP can be {bounded using only} $\pi_0$, and $\pi_n$ without recomputing $\pi_{n, c}$ over every label $c$. \autoref{fig:section3} evaluates this section exactly on a conjugate Bayesian linear regression (setup in \autoref{app:blr-check}), where all three posteriors are Gaussian and every term of \autoref{eq:bayesian-margin-decomposition} is a closed-form moment: the identities of \autoref{lem:formability-variance} and \autoref{thm:full-split-decomposition}(i) hold to machine precision, the full margin computed from the actual augmented posterior stays inside the band for every candidate, and the band is tight exactly where the calibration neighborhood is dense.

\parbold{Consistent Bayesian learner.} By increasing $n$ (calibration set size) the posterior $\pi_n$ concentrates. \autoref{lem:formability-variance}, and \autoref{thm:full-split-decomposition} {tie} formability, and quantile change to exactly the disagreement between plausible parameters from $\pi_n$. The concentration {shrinks} $\Phi_{n,c}$ and $T_{n,c}$ and the full margin $m_{n, c}$ converges to the calibration-posterior margin $m_{n,c}\to m_0+A_{n,c}$ which is also observed in \autoref{fig:full-vs-split-calib-pca}.

\subsection{Deterministic Fine-Tuning}
\label{sec:map-refits}

\autoref{thm:full-split-decomposition}, and \autoref{lem:formability-variance} are written for Bayesian learning (given posterior $\pi_n$), which is intractable to compute for deep models. A common regime for deterministic fine-tuning finds the maximizer of the (log-) posterior, i.e., $\hat\vtheta_n:=\arg\max_\vtheta\gL(\vtheta)$, for $\gL(\vtheta):=\log\pi_0(\vtheta)+\sum_{i=1}^n\log p_\vtheta(y_i\mid \vx_i)$; e.g. the common $\ell_2$-regularized objective assigns a Gaussian prior $\pi_0(\vtheta)=\gN(\vtheta_0,\tau^2\mI)$.
\citet{daxberger2021laplace} {approximate} $\pi_n$ with a Gaussian distribution $\gN(\hat\vtheta_n, \hat{\mSigma}_n)$ centered {at the MAP} $\hat\vtheta_n$, and covariance given by the inverse Hessian of the log-posterior; i.e. $\mSigma = -\nabla^2_\vtheta\gL(\hat\vtheta_n)^{-1} |_{\hat\theta}$, which
approximates the variance of the \emph{predictive} distribution as $V_{n,c} \approx \hat{V}^{\mathrm L} := \nabla_\vtheta p_{\hat\vtheta_n}(c\mid \xtest)^\top \hat{\mSigma}_n \nabla_\vtheta p_{\hat\vtheta_n}(c\mid \xtest)$. 
This approximation {with a Gaussian} distribution is the result {of a second-order} Taylor expansion of $\gL$ around $\hat\vtheta_n$.
By construction, a shortcoming is that this approximation is always uni-modal (Gaussian), regardless of the shape of the true posterior (which can be multi-modal, skewed, etc.). In other words, there are dissimilar distributions (with different variances) that are all approximated by the same Gaussian distribution under this approach. Aware of this approximation scheme, \autoref{prop:laplace-blind} designs such dissimilar distributions.

\begin{restatable}{proposition}{laplaceblind}
\label{prop:laplace-blind}
Let $\pi_n$ be twice differentiable at $\hat\vtheta_n$, and $U$ be an open neighborhood of $\hat\vtheta_n$. Define $\eta:=\pi_n(\hat\vtheta_n)-\sup_{\vtheta\notin U}\pi_n(\vtheta)>0$. Let $\pi_\mathrm{res}$ be a probability density that vanishes on $U$, and for $\delta\in(0,1)$ satisfy $\delta\|\pi_\mathrm{res}\|_\infty<(1-\delta)\eta$, and set $\pi_n':=(1-\delta)\pi_n+\delta\pi_\mathrm{res}$. Then:

(i) Every maximizer of $\pi_n'$ lies in $U$ and is a maximizer of $\pi_n${. In} particular $\hat\vtheta_n$ maximizes $\pi_n'$, and $\hat V^{\mathrm L}(\pi_n')=\hat V^{\mathrm L}(\pi_n)$.

(ii) With $w_{\mathrm{res},c}:=\E_{\vtheta\sim\pi_\mathrm{res}}[p_\vtheta(c\mid\xtest)]$ and $V_{\mathrm{res},c}:=\Var_{\vtheta\sim\pi_\mathrm{res}}[p_\vtheta(c\mid\xtest)]$,
\begin{align}
    V_{n,c}(\pi_n')-V_{n,c}(\pi_n)
    &\geq\delta\bigl[(1-\delta)(w_{\mathrm{res},c}-w_{n,c})^2-V_{n,c}(\pi_n)\bigr].
    \label{eq:laplace-blind}
\end{align}
\end{restatable}

Part (i) shows that the Laplace approximation of the different distributions remains the same while (ii) the variance of the predictive distribution can differ. This is well known, and it is what motivates mixing several Laplace fits \citep{eschenhagen2021mola}, ensembling across basins \citep{wilson2020bayesian}, or searching directly for plausible parameters that disagree about the prediction \citep{schweighofer2023quam}. 
{Basically, the} issue stems from the fact that the Bayesian posterior distribution can have multiple basins, and basins other than the one with $\hat\theta_n$ can {contribute} meaningfully to the predictive variance. {Formability retrains} the model over each candidate label, meaning that it can jump to the basin that is more likely to predict the candidate label. The following proposition shows that {the refit can result in a maximum a posteriori estimate outside} of the basin of the original MAP.

\begin{restatable}{proposition}{mapcandidateselection}
\label{prop:map-candidate-selection}
Let $\gB_0$ be a basin containing $\hat\vtheta_n$, and suppose that for some $r\geq0$ and $0<u<v\leq1$, and some $\theta'\notin \gB_0$ outside {of the basin} containing the MAP estimate, we have \[
\sup_{\theta \in \gB_0} p_\theta(c\mid\xtest) \le u < v \le p_{\theta'}(c\mid\xtest), \quad \text{and} \quad \gL(\hat\theta_n) - \gL(\theta') \le  r.
\]
Then, if $\log(v/u) > r$, the augmented MAP $\hat\theta_{n,c} \notin \gB_0$. In either case $p_{\hat\theta_{n,c}}(c\mid\xtest) \ge e^{-r} v$.
\end{restatable}

\section{Experiments}
\label{sec:experiments}

\begin{figure}[t]
  \centering
  \includegraphics{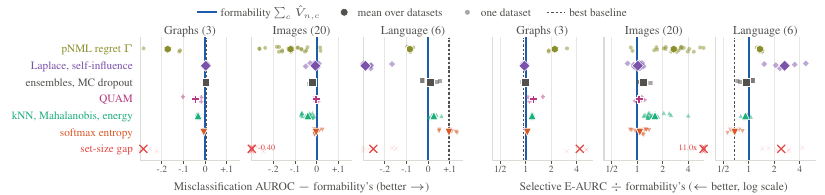}
  \caption{
    Formability against families of uncertainty {measures, for in-distribution} data, on graphs, images and language. For each modality the relative performance is gathered from different datasets and models. Each family of {methods includes} several variations (e.g. Laplace is last-layer and all-parameter linearised, last-layer sampled; self-influence). A detailed comparison is in \autoref{fig:leaderboard-grouped-all}.
  }
  \label{fig:leaderboard-grouped}
\end{figure}

With our experiments we aim to show that 
(i) the performance gap between full and split conformal prediction is governed by the size of the calibration set, training set, and the dimension of the feature space; and that 
(ii) it {stems from the} model's formability which is directly connected to epistemic uncertainty (synthetic experiments).
(iii) The formability itself can serve as a tool to quantify epistemic uncertainty, through the variance as computed in \autoref{eq:total-variance-from-formability}, $\hat V_{n, c} = \hat{w}_{n, c} \hat\Phi_{n, c}$.
We compare formability against several existing uncertainty quantification methods: softmax entropy \citep{hendrycks2017baseline}, deep ensembles \citep{lakshminarayanan2017simple}, MC-dropout \citep{gal2016dropout}, Laplace approximation \citep{daxberger2021laplace}, and DeepKNN \citep{papernot2018deep}. The closest {methods to formability are pNML regret \citep{bibas2019deep}, and the set size gap, against which we specifically compare, together with} the best of the aforementioned methods in \autoref{fig:pnml-regret-main}.
Reporting against the {best does} not mean that there is a single best method at every level, but rather every method has its own sweet spot. {Overall} relative performance is reported in \autoref{fig:leaderboard-grouped}, and detailed reports are in \autoref{app:supplementary-results}.

For language models, refitting together over a calibration set and test point {for each} pseudo-label becomes unrealistically expensive. Therefore we approximate the refit with a single step of LoRA \citep{Hu2021LoRALA} fine-tuning. Our updates are done through damped curvature-informed steps, and for calibration points we cache the gradients.
We compare methods on six four-option question-answering setups (MMLU, MedQA on small Qwen, and Llama models). 
The specifics of the setup are in \autoref{app:llm-experiments}. Note that the language modeling setup is defined in a very {simplistic} way. The model is simply asked to choose {one of the A--D} answers and the score is the softmax over those specific tokens (see \autoref{app:llm-experiments}). The model updates are over a LoRA that is applied to the query and value matrices of the attention layers. These are specifically chosen for the computation to fit on a single GPU. The results are also not comparable to the image and graph setups. Defining a clear formability signal for language models is an open question, and we leave it for future work.

Throughout this paper we assumed the TPS score function{ \citep{sadinle2019least}} $s(\vx,y) = p_\vtheta(y\mid \vx)$, and $1-\alpha=0.9$ unless otherwise specified. While {CP} is commonly phrased with non-conformity (disagreement) scores, we use agreement as it simplifies the derivations, and interpretation. Setups are equivalent through a change in sign. For uncertainty quantification, we use two main metrics: (i) the misclassification AUROC{ \citep{hendrycks2017baseline}} which is the {area under the curve of} true positive rate vs false positive rate over rejection according to the uncertainty score function; 0.5 is {chance} and 1.0 is the perfect score.
(ii) The selective E-AURC{ \citep{geifman2019biasreduced}} is {the difference in} AURC between the uncertainty score, and the ideal uncertainty score (the oracle). In this score lower is better, and 0 is the best score. 

\parbold{Full-vs-Split CP.} In \autoref{fig:full-vs-split-calib-pca} we compare full and split conformal prediction on node  and image-classification tasks. For {dimensionality reduction,} we apply PCA on input features. Specifically on \texttt{CIFAR-10}, the features are the embeddings from {an \texttt{ImageNet}-pretrained} ResNet-18 \citep{he2016deep} backbone, (training is done on the {two-layer} MLP as classification head). Since full CP is computationally expensive to examine, we report these results on a random subset of the test set containing $300\sim400$ instances. Prior to observing test examples, we can approximate the full CP set size by a small (e.g. $m=50$) number of calibration {points over leave-one-out} refits (keeping $n-1$ points as calibration set and the left out point as test set). We run ablation study and show in \autoref{fig:train-calib-grid}, that by increasing calibration size, training size, or decreasing the dimension of the feature space, the full CP outperforms split CP, and it converges to the adaptation lower bound $|\gC_B|$ (consistent with \autoref{thm:full-split-decomposition}). With a larger training set, the effect of calibration set shows a diminishing return.

\parbold{Synthetic experiment.} In \autoref{fig:score-anatomy} we visualize the concept discussed in \autoref{sec:formability}. If a point is far from the calibration set it can deform the model more easily, its score can increase more, and full conformal set size would also inflate. It also shows the three variables {defined in} \autoref{eq:margin-decomposition}. To synthesize the data, {we sample} three classes from overlapping Gaussians, and use a multinomial logistic regression over fixed RBF features as the model. The pretrained model $\vtheta_0$ is trained on only $3$ points per class, and every fine-tuning minimizes {$\ell_2$-regularized cross-entropy towards} $\vtheta_0$ (i.e. the MAP objective with the prior $\gN(\vtheta_0,\tau^2\mI)$). We fine-tune once on the calibration set ($n=42$) to obtain $p_n$, and once per candidate label $c$ on $\daugmented$ for $p_{n,c}$. The complete anatomy for every candidate label {is in} \autoref{app:score-anatomy}. An important extension of this experiment is shown in \autoref{fig:epistemic-sweeps}, where we move the Gaussian generators {towards the center} which reduces epistemic uncertainty, and increases aleatoric uncertainty at the center. Here the softmax entropy remains steady, while the formability variance decreases. This shows that formability {directly distinguishes} the epistemic uncertainty, while the softmax entropy is agnostic to this decomposition.   

In \autoref{fig:section3} we evaluate \autoref{thm:full-split-decomposition} in a setup where every posterior can be computed exactly. We sample from $y=\sin(1.5x)+0.3x+\varepsilon$, with Gaussian noise of standard deviation $0.5$, and use Bayesian linear regression over fixed RBF features. The model is pretrained on $40$ points in $[-4,-0.5]$, while the $60$ calibration points concentrate around $x=0$. The full sets are narrower than split, especially where calibration data is dense. In sparse regions, the posterior variance and formability are larger, and the full set expands relative to the model fitted on calibration data alone. The full margins remain within the theoretical bounds at both test points. Details are in \autoref{app:blr-check}.

\parbold{Formability as a quantification of epistemic uncertainty.} While {refitting on every label as} an uncertainty quantification is {not an entirely new} concept (\citet{bibas2019deep} {define} regret based on refitting to every candidate label), we connect the refit through formability to Bayesian posterior variance, and the uncertainty captured in full conformal prediction. 
While {both we} and \citet{bibas2019deep} take information from {models fitted on pseudo-labels}, they take the logarithm-sum of the candidate probabilities, namely, {$\log\bigl(1 + \sum_{c\in\gY} \Phi_{n,c}\bigr)$}, while we directly estimate the posterior variance $\Phi_{n,c}\cdot w_{n,c}$. This is derived directly from CP's uncertainty. Interestingly, as shown in \autoref{fig:pnml-regret-main}, this minor change in the {definition can} result in a different evaluation on the data. Formability performs significantly better on in-domain error {detection}, while the regret-based uncertainty {performs better on out-of-distribution} data. We compare all methods briefly in \autoref{fig:leaderboard-grouped}. {The CP set} size gap (full - split) itself is not a good uncertainty quantification; mainly because it quantizes all the information captured in \autoref{eq:total-variance-from-formability}, which is a continuous value, into a binary decision {per} label.

\begin{figure}[t]
  \centering
  \includegraphics{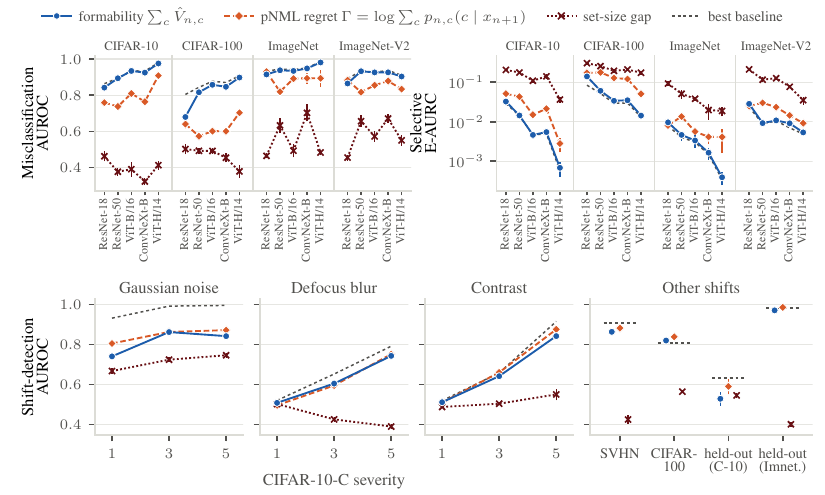}
  \caption{
    Comparison of formability, regret and conformal set size gap over image benchmarks.
    }
  \label{fig:pnml-regret-main}
\end{figure}

\section{Conclusion}

We investigated the performance gap between full, and split conformal prediction. We showed that the gap is the resultant of two opposite effects: adapting the model on the calibration set (more data resulting in better model, and hence smaller sets), and forming it towards each pseudo-label at the test point. Under Bayesian learning, the second effect (formability) is directly tied to posterior predictive variance, and fades as the calibration posterior concentrates. This explains why full CP can become more efficient when the model is already well determined around the test point, and in sparse regions it performs poorly. However, this gap suggest to use formability as a measure of epistemic uncertainty that carries much more information compared to set size (due to the 0/1 binarization). Our experiments on images and graphs show its value for detecting prediction errors; and to some degree, ood detection. In many benchmarks, formability is on par with the best of the existing uncertainty quantification methods, and in some cases it outperforms them.

\parbold{Limitations.} The most important limitation is the computational cost of refitting the model on every candidate label. Additionally our analysis, and definition of formability is based on the softmax score function. CP itself is agnostic to the choice of score function, and scores like APS also can lead to meaningful interpretations of the full-vs-split gap. We leave this for future work. Additionally, conformal prediction assumes exchangeability of the data, which is not always satisfied in practice, specifically for uncertainty quantification use cases. Formability implicitly assume that refits are able to reach the MAP estimate. This could challenge the quality in overparameterized models, or in wrong early stopping scenarios{. Potentially} our language model experiments are affected by this. 

\newpage 
\subsection*{AI use statement}

We used generative AI tools for implementing the experiments (we designed the experimental setup ourselves and asked AI to implement and run them. We also inspected the {results for} potential issues). We also used AI for rephrasing parts of the paper, from grammar and choice of words, to better presenting the proofs. The proofs are ultimately rewritten by the authors, but AI was used to assist in the process.
The research idea was originally proposed by the authors, including the main idea, theory, and the primary design of the experiments. 

\subsection*{Ethics statement}

This work does not raise any ethical concerns, as it is a general study, and a theoretical contribution to uncertainty quantification. The experiments are conducted on publicly available datasets, and the models are trained on these datasets. The authors have no conflicts of interest to disclose.

\subsection*{Reproducibility statement}

All datasets, and models used in the experiments are publicly available. The code for reproducing the experiments is publicly available at \url{https://github.com/soroushzargar/Formability}. 

\bibliography{iclr2027_conference}
\bibliographystyle{iclr2027_conference}

\appendix
\newpage

\section{Related Works}
\label{sec:related-works}
\parbold{Conformal prediction.} CP \citep{vovk2005algorithmic,angelopoulos2024theoretical} is {considered a} distribution-free uncertainty quantification technique, where the prediction sets are guaranteed to cover the true label with high probability. With that, many consider the size of the prediction {set to represent} how uncertain the model is on a specific input. While the method works with any model, and any score function, there are many works designing various scores to improve the set size efficiency, adaptivity, and approximation of conditional coverage {\citep{romano2020classification,Angelopoulos2020UncertaintySF,zargarbashi2023conformal}}. We use the basic threshold prediction sets \citep{sadinle2019least} since it directly connects to the predictive distribution of the model which can connect perfectly with our decomposition in \autoref{sec:formability}, and through which we bound the epistemic uncertainty{. In} \autoref{app:train-vs-calib-aps} we repeat our experiments under an adaptive (rank-based) score, where the full--split gap is wider, and never closes. While full and split conformal prediction are two endpoints of a computational cost spectrum, there are other variants such as Jackknife+ \citep{barber2021jackknife} in between. There are techniques that approximate full CP with {lower computational cost \citep{bhatt2021fast,martinez2023approximating}}, through influence functions \citep{koh2017influence}. Both directions are outside of our focus, as we address the fundamental difference between full and split CP.

\parbold{CP and epistemic uncertainty.} \citet{javanmardi2026optimal,cabezas2025epicscore} {incorporate a second-order} distribution (which contains explicit epistemic uncertainty) {into conformal prediction}. {In the reverse direction, \citet{chau2026quantifying,correia2024information}} provide interpretations of split conformal prediction in uncertainty quantification; one through connecting split CP to imprecise integral probability metrics, and the other through information theory. In contrast, we provide a direct connection between the difference of full and split CP set size {and the epistemic} uncertainty, by bounding the posterior variance of the predictive distribution. 

Our work connects three arguments: \begin{enumerate*}[label=(\roman*)]
    \item When does full CP outperform split CP? Inconsistencies in comparison are mentioned by \citet{barber2023conformal} and extensively studied by \citet{zong2025comparative}. However, {neither work discusses what the underlying factor is} that governs this difference. We show that the unexpected inefficiency in full CP set size stems from the formability of the model at the test point. 
    \item We show that, in a Bayesian framework, the formability is directly related to the posterior variance (epistemic uncertainty). The concept of \emph{deterministic} formability itself is {not} entirely new. \citet{bibas2019deep} {define} a learnability measure (regret) as the log sum of the predictive probabilities $\hat\theta_{n, c}$ normalized by the probabilities from $\hat\theta_{n}$. However, its connection with Bayesian posterior variance, comparison with deterministic Laplace approximation through basin decomposition \citep{daxberger2021laplace}, or its influence on full conformal prediction do not follow automatically.
    \item We show that on larger models and various tasks this notion of formability still can be used as a practical tool to quantify epistemic uncertainty. 
\end{enumerate*}
Our work also connects with many other uncertainty quantification approaches. While the posterior distribution of parameters is a canonical way to quantify epistemic uncertainty, on larger models it is computationally near-impossible to compute. Laplace Redux \citep{daxberger2021laplace} approximates this distribution as a Gaussian around the MAP estimate{. However}, as we show in \autoref{prop:laplace-blind} it is agnostic to other basins of attraction, and therefore miscalculates the epistemic uncertainty where other highly {plausible basins disagree with the current predictive distribution}. A {remedy to that is MC-dropout \citep{gal2016dropout}}, and ensemble training {\citep{lakshminarayanan2017simple}}, which are parameter centric approximations. Formability however is a data-centric approximation. On {computation, the comparison depends on whether} multiple model inference or test-time training is more favorable. \citet{schweighofer2023quam} instead searches for adversarial models with high disagreement to the baseline, while attaining high posterior plausibility. In comparison, formability is a local measure, searching for such models through retraining. The comparison results as shown in \autoref{fig:cost-quality} {are very close}. Our ultimate goal is also not to come up with a new {notion} of uncertainty, but to showcase what is the underlying uncertainty quantification in conformal prediction.

\section{Supplementary to Theory}

We first prove the sequential updating property used in \autoref{sec:theory}.

\begin{lemma}[Sequential Bayesian updating]
\label{lem:sequential-updating}
Let $\pi_n$ be the posterior from the prior $\pi_0$ and evidence $\dcal$, and $\pi_{n,c}$ be the posterior from the prior $\pi_0$ and evidence $\daugmented$. For a fixed $\vx_{n+1}$, assume its label to be independent of $\dcal$ given $\vtheta$, and $w_{n,c}=\E_{\vtheta\sim\pi_n}[p_\vtheta(c\mid \vx_{n+1})]>0$ (recalling \autoref{eq:candidate-mass}). Then $\pi_{n,c}$ {equals} the posterior computed from the prior $\pi_n$ on the single observation $\{(\vx_{n+1},c)\}$; i.e.,
\begin{align}
    \pi_{n,c}(\vtheta)
    =\frac{p_\vtheta(c\mid \vx_{n+1})\pi_n(\vtheta)}
    {\E_{\vtheta\sim\pi_n}[p_\vtheta(c\mid \vx_{n+1})]}
    =\frac{p_\vtheta(c\mid \vx_{n+1})\pi_n(\vtheta)}{w_{n,c}}.
    \label{eq:appendix-candidate-posterior}
\end{align}
\end{lemma}

\begin{proof}
For two sets of evidence $E_1$ and $E_2$, Bayes' rule with the joint evidence gives
\begin{align*}
    p(\vtheta\mid E_1,E_2)
    =\frac{p(E_1,E_2\mid\vtheta)\,\pi_0(\vtheta)}{p(E_1,E_2)}
    =\frac{p(E_2\mid\vtheta,E_1)\,p(E_1\mid\vtheta)\,\pi_0(\vtheta)}{p(E_2\mid E_1)\,p(E_1)},
\end{align*}
Bayes' rule on $E_1$ alone yields $p(E_1\mid\vtheta)\,\pi_0(\vtheta)=p(\vtheta\mid E_1)\,p(E_1)$, and the law of total probability expands $p(E_2\mid E_1)=\int p(E_2\mid\vartheta,E_1)\,p(\vartheta\mid E_1)\,d\vartheta$. Substituting both,
\begin{align}
    p(\vtheta\mid E_1,E_2)
    =\frac{p(E_2\mid\vtheta,E_1)\,p(\vtheta\mid E_1)}
    {\int p(E_2\mid\vartheta,E_1)\,p(\vartheta\mid E_1)\,d\vartheta}.
    \label{eq:sequential-bayesian-updating}
\end{align}
{This means that} the posterior after observing $E_1$ can be used as the prior when conditioning on $E_2$, and it results in the same posterior as conditioning jointly on both sets of evidence.

In our case, $E_1=\dcal$ and $E_2=\{(\vx_{n+1},c)\}$, therefore $p(\vtheta\mid E_1)=\pi_n(\vtheta)$ and $p(\vtheta\mid E_1,E_2)=\pi_{n,c}(\vtheta)$. With $\vx_{n+1}$ fixed, 
 we have $p(E_2\mid\vtheta,E_1)=p_\vtheta(c\mid \vx_{n+1})$, and the denominator of \autoref{eq:sequential-bayesian-updating} becomes $\E_{\vtheta\sim\pi_n}[p_\vtheta(c\mid \vx_{n+1})]=w_{n,c}$, which is positive by assumption. 
\end{proof}

\formabilityvariancelemma*

\begin{proof}
Using \autoref{eq:appendix-candidate-posterior},
\begin{align*}
    \E_{\vtheta\sim\pi_{n,c}}[p_\vtheta(c\mid \vx_{n+1})]
    &=\int p_\vtheta(c\mid \vx_{n+1})\pi_{n,c}(\vtheta)\,d\vtheta
    \\&= \int p_\vtheta(c\mid \vx_{n+1})
    \frac{p_\vtheta(c\mid \vx_{n+1})\pi_n(\vtheta)}{w_{n,c}}\,d\vtheta
    =\frac{
        \E_{\vtheta\sim\pi_n}[p_\vtheta(c\mid \vx_{n+1})^2]
    }{w_{n,c}}.
\end{align*}
{Substituting} it into \autoref{eq:formability} yields
\begin{align*}
    \Phi_{n,c}(\vx_{n+1})
    &=\frac{
        \E_{\vtheta\sim\pi_n}[p_\vtheta(c\mid \vx_{n+1})^2]
    }{w_{n,c}} - w_{n, c}\\
    &=\frac{
        \E_{\vtheta\sim\pi_n}[p_\vtheta(c\mid \vx_{n+1})^2]
        -w_{n,c}^2
    }{w_{n,c}}
    =\frac{
        \Var_{\vtheta\sim\pi_n}[p_\vtheta(c\mid \vx_{n+1})]
    }{w_{n,c}}.
\end{align*}
Multiplying both sides by $w_{n,c}$ results in \autoref{eq:variance-from-formability}.

For the remainder of the proof, we will use the following rule: for any $X \in [0, 1]$ we have \begin{align*}
    \Var(X) = \E[X^2] - (\E[X])^2 \leq \E[X] - (\E[X])^2 = \E[X](1 - \E[X]).
\end{align*} As $0\leq p_\vtheta(c\mid \vx_{n+1})\leq1$, its posterior variance
is bounded by
\begin{align*}
    \Var_{\vtheta\sim\pi_n}[p_\vtheta(c\mid \vx_{n+1})]
    \leq w_{n,c}(1-w_{n,c}).
\end{align*}

Replacing the variance in {\autoref{eq:variance-from-formability}} with this inequality gives $\Phi_{n,c}(\vx_{n+1})\leq1-w_{n,c}$, and hence
$w_{n,c}\leq1-\Phi_{n,c}(\vx_{n+1})$. Therefore,
\begin{align*}
    \Var_{\vtheta\sim\pi_n}[p_\vtheta(c\mid \vx_{n+1})]
    &=w_{n,c}\Phi_{n,c}(\vx_{n+1})\\
    &\leq\Phi_{n,c}(\vx_{n+1})
    \bigl(1-\Phi_{n,c}(\vx_{n+1})\bigr).
\end{align*}

For a finite number of labels, the expected squared Euclidean norm is
\begin{align*}
    &\E_{\vtheta\sim\pi_n}\!\left[
        \left\|p_\vtheta(\cdot\mid \vx_{n+1})
        -p_{\pi_n}(\cdot\mid \vx_{n+1})\right\|_2^2
    \right]=\sum_{c\in\gY}
    \E_{\vtheta\sim\pi_n}\!\left[
        \left(p_\vtheta(c\mid \vx_{n+1})-w_{n,c}\right)^2
    \right]\\
    &\quad=\sum_{c\in\gY}
    \Var_{\vtheta\sim\pi_n}[p_\vtheta(c\mid \vx_{n+1})]\\
    &\quad=\sum_{c\in\gY}w_{n,c}\Phi_{n,c}(\vx_{n+1})\\
    &\quad\leq\sum_{c\in\gY}\Phi_{n,c}(\vx_{n+1})
    \bigl(1-\Phi_{n,c}(\vx_{n+1})\bigr),
\end{align*}
which proves {\autoref{eq:total-variance-from-formability}}.

\end{proof}

\subsection{Proof of \autoref{thm:full-split-decomposition}}

Before proving \autoref{thm:full-split-decomposition}, we first prove two lemmas:
\begin{enumerate*}[label=(\roman*)]
    \item \autoref{lem:candidate-conditioning} shows how conditioning on the pseudo-labeled test point moves the posterior expectation w.r.t. the expectation {under the prior}. In our case the posterior is the augmented posterior, {and the prior} is the calibration posterior. 
    \item \autoref{lem:quantile-shift} records the stability of the conformal quantile under coordinatewise shifts of the scores.
\end{enumerate*}
The proof relies on two auxiliary lemmas: the first shows how conditioning on the pseudo-labeled test point moves \emph{any} posterior expectation, and the second records the stability of the conformal quantile under coordinatewise shifts of the scores.

\begin{lemma}
\label{lem:candidate-conditioning}
For any function $h:\Theta\to\R$ with $\E_{\vtheta\sim\pi_n}|h(\vtheta)|<\infty$,
\begin{align}
    \E_{\vtheta\sim\pi_{n,c}}[h(\vtheta)]
    =\E_{\vtheta\sim\pi_n}[h(\vtheta)]
    +\frac{\Cov_{\vtheta\sim\pi_n}\bigl(h(\vtheta),\,p_\vtheta(c\mid \vx_{n+1})\bigr)}{w_{n,c}}.
    \label{eq:candidate-conditioning}
\end{align}
\end{lemma}

\begin{proof}
Using \autoref{eq:appendix-candidate-posterior},
\begin{align*}
    \E_{\vtheta\sim\pi_{n,c}}[h(\vtheta)]
    =\int h(\vtheta)\frac{p_\vtheta(c\mid \vx_{n+1})\pi_n(\vtheta)}{w_{n,c}}\,d\vtheta
    =\frac{\E_{\vtheta\sim\pi_n}\bigl[h(\vtheta)\,p_\vtheta(c\mid \vx_{n+1})\bigr]}{w_{n,c}}.
\end{align*}
By the definition of covariance,
\begin{align*}
    \E_{\vtheta\sim\pi_n}\bigl[h(\vtheta)\,p_\vtheta(c\mid \vx_{n+1})\bigr]
    &=\E_{\vtheta\sim\pi_n}[h(\vtheta)]\,
    \E_{\vtheta\sim\pi_n}[p_\vtheta(c\mid \vx_{n+1})]\\
    &\quad+\Cov_{\vtheta\sim\pi_n}\bigl(h(\vtheta),\,p_\vtheta(c\mid \vx_{n+1})\bigr),
\end{align*}
and $\E_{\vtheta\sim\pi_n}[p_\vtheta(c\mid \vx_{n+1})]=w_{n,c}$. Dividing by $w_{n,c}$ completes the proof.
\end{proof}

Note that \autoref{eq:variance-from-formability} in \autoref{lem:formability-variance} is the special case $h(\vtheta)=p_\vtheta(c\mid \vx_{n+1})$, since the covariance of a variable with itself is its variance.

\begin{lemma}[Quantile stability under score shifts]
\label{lem:quantile-shift}
The map $r_{1:n}\mapsto Q_{\alpha,n}(r_{1:n})$ defined in \autoref{eq:lower-conformal-quantile} is non-decreasing in every coordinate, and translation equivariant whenever the quantile is finite, i.e., $Q_{\alpha,n}(r_{1:n}+t\1)=Q_{\alpha,n}(r_{1:n})+t$ for every $t\in\R$. For any shift vector $\delta_{1:n}$,
\begin{align}
    Q_{\alpha,n}(r_{1:n})+\min_i\delta_i
    \;\leq\;Q_{\alpha,n}(r_{1:n}+\delta_{1:n})
    \;\leq\;Q_{\alpha,n}(r_{1:n})+\max_i\delta_i.
    \label{eq:quantile-sandwich}
\end{align}
\end{lemma}

\begin{proof}
$Q_{\alpha,n}(r_{1:n})$ is the $\alpha$-quantile of the empirical distribution over the multiset $\{r_1,\dots,r_n,\infty\}$, and hence returns the $k$-th smallest element of this multiset, where the rank $k$ is fixed by $\alpha$ and $n$ (with $Q_{\alpha,n}=\infty$ when $k=n+1$; we assume $\alpha$ and $n$ are such that the quantile is finite). Increasing any coordinate cannot decrease any order statistic, which gives monotonicity, and adding the same $t$ to every score adds $t$ to every finite order statistic, which gives translation equivariance. For the sandwich, note that coordinatewise
$r_i+\min_j\delta_j\leq r_i+\delta_i\leq r_i+\max_j\delta_j${. Applying} monotonicity and then translation equivariance to both sides yields \autoref{eq:quantile-sandwich}.
\end{proof}

\fullsplitdecomposition*

\begin{proof}
\textbf{Part (i).} Under Bayesian training, the calibration scores are posterior predictive probabilities: $r_i^{(n)}=p_{\pi_n}(y_i\mid \vx_i)=\E_{\vtheta\sim\pi_n}[p_\vtheta(y_i\mid \vx_i)]$ and $r_i^{(c)}=p_{\pi_{n,c}}(y_i\mid \vx_i)=\E_{\vtheta\sim\pi_{n,c}}[p_\vtheta(y_i\mid \vx_i)]$. Applying \autoref{lem:candidate-conditioning} with $h(\vtheta)=p_\vtheta(y_i\mid \vx_i)$ gives
\begin{align*}
    r_i^{(c)}-r_i^{(n)}
    =\frac{\Cov_{\vtheta\sim\pi_n}\bigl(p_\vtheta(y_i\mid \vx_i),\,p_\vtheta(c\mid \vx_{n+1})\bigr)}{w_{n,c}}
    =\frac{C_{i,c}}{w_{n,c}}.
\end{align*}
The bound $|C_{i,c}|\leq\sqrt{V_i V_{n,c}}$ is the Cauchy--Schwarz inequality for the covariance.

\textbf{Part (ii).} We first bound the threshold drift. By definition, $T_{n,c}=q_\vtheta^{(c)}-q_\vtheta^{(n)}$ with $q_\vtheta^{(c)}=Q_{\alpha,n}(r_{1:n}^{(c)})$ and $q_\vtheta^{(n)}=Q_{\alpha,n}(r_{1:n}^{(n)})$. Defining $\delta_i$ as the difference in the scores {after the update}, we have $r_{1:n}^{(c)}=r_{1:n}^{(n)}+\delta_{1:n}$ with $\delta_i=C_{i,c}/w_{n,c}$, so \autoref{lem:quantile-shift} gives
\begin{align}
    \frac{\min_i C_{i,c}}{w_{n,c}}\leq T_{n,c}\leq\frac{\max_i C_{i,c}}{w_{n,c}},
    \qquad\text{hence}\qquad
    |T_{n,c}|\leq\frac{\max_i|C_{i,c}|}{w_{n,c}}\leq\frac{\sqrt{V_{n,c}}\max_i\sqrt{V_i}}{w_{n,c}},
    \label{eq:threshold-drift-sandwich}
\end{align}
where the last step is the Cauchy--Schwarz bound of part (i).

Next, \autoref{eq:margin-decomposition} is an algebraic identity that holds for any three fits:
$m_{n,c}=m_0+A_{n,c}+\Phi_{n,c}-T_{n,c}$. Under Bayesian refitting where the formability is defined as \autoref{eq:formability}, $\Phi_{n,c}=p_{\pi_{n,c}}(c\mid \vx_{n+1})-p_{\pi_n}(c\mid \vx_{n+1})$, we write $\Phi_{n, c} - T_{n, c} = \frac{V_{n,c}}{w_{n,c}} - T_{n, c}$ (given $\Phi_{n,c}=V_{n,c}/w_{n,c}$). Therefore, $m_{n,c}=m_0+A_{n,c}+V_{n,c}/w_{n,c}-T_{n,c}$. By \autoref{eq:threshold-drift-sandwich}, we have
\begin{align*}
    \frac{V_{n,c}}{w_{n,c}}-T_{n,c}
    &\in\left[\frac{V_{n,c}-\sqrt{V_{n,c}}\max_i\sqrt{V_i}}{w_{n,c}},\;
    \frac{V_{n,c}+\sqrt{V_{n,c}}\max_i\sqrt{V_i}}{w_{n,c}}\right]\\
    &=\left[\frac{\sqrt{V_{n,c}}\bigl(\sqrt{V_{n,c}}\pm\max_i\sqrt{V_i}\bigr)}{w_{n,c}}\right]
    =\bar\Phi_{n,c},
\end{align*}
which proves \autoref{eq:bayesian-margin-decomposition}.
\end{proof}

\subsection{Proof of \autoref{prop:laplace-blind}}

\laplaceblind*

\begin{proof}
(i) On $U$, $\pi_\mathrm{res}=0$, so $\pi_n'=(1-\delta)\pi_n\leq(1-\delta)\pi_n(\hat\vtheta_n)=\pi_n'(\hat\vtheta_n)$, which means that $\pi_n'$ is maximized at the same points as $\pi_n$ in $U$. Outside $U$,
\begin{align*}
    \pi_n'(\vtheta)\leq(1-\delta)\bigl(\pi_n(\hat\vtheta_n)-\eta\bigr)+\delta\|\pi_\mathrm{res}\|_\infty<(1-\delta)\,\pi_n(\hat\vtheta_n)=\pi_n'(\hat\vtheta_n).
\end{align*}
Hence, the maximizers of $\pi_n'$, and $\pi_n$ are {the same} in $U$, and $\hat\vtheta_n$ is one of them. Since $\log\pi_n'=\log(1-\delta)+\log\pi_n$ on the open set $U$ (meaning that the distribution is a scaled {version of $\pi_n$} in that region), the two log-densities have the same Hessian at $\hat\vtheta_n$, hence the same $\hat\mSigma_n$ (the $\log(1-\delta)$ disappears {when differentiating}). The gradient $\nabla_\vtheta p_{\hat\vtheta_n}(c\mid\xtest)$ depends only on the model and on $\hat\vtheta_n$. Therefore $\hat V^{\mathrm L}(\pi_n')=\hat V^{\mathrm L}(\pi_n)$.

(ii) Write $f(\vtheta):=p_\vtheta(c\mid\xtest)$. Expectations are linear in the density, so
\begin{align*}
    \E_{\pi_n'}[f]=(1-\delta)\,w_{n,c}+\delta\,w_{\mathrm{res},c},
    \qquad
    \E_{\pi_n'}[f^2]=(1-\delta)\bigl(V_{n,c}(\pi_n)+w_{n,c}^2\bigr)+\delta\bigl(V_{\mathrm{res},c}+w_{\mathrm{res},c}^2\bigr).
\end{align*}
Subtracting the square of the first from the second,
\begin{align*}
    V_{n,c}(\pi_n')
    &=(1-\delta)\,V_{n,c}(\pi_n)+\delta\,V_{\mathrm{res},c}+(1-\delta)\,w_{n,c}^2+\delta\,w_{\mathrm{res},c}^2-\bigl((1-\delta)\,w_{n,c}+\delta\,w_{\mathrm{res},c}\bigr)^2, \\
    &=(1-\delta)\,V_{n,c}(\pi_n)+\delta\,V_{\mathrm{res},c}+\delta(1-\delta)(w_{\mathrm{res},c}-w_{n,c})^2.
\end{align*}
where the last three terms equal $\delta(1-\delta)\,w_{n,c}^2+\delta(1-\delta)\,w_{\mathrm{res},c}^2-2\delta(1-\delta)\,w_{n,c}w_{\mathrm{res},c}=\delta(1-\delta)(w_{\mathrm{res},c}-w_{n,c})^2$. Subtracting $V_{n,c}(\pi_n)$ from both sides results in \autoref{eq:laplace-blind}, and the inequality follows from $V_{\mathrm{res},c}\geq0$.
\end{proof}

\subsection{Proof of \autoref{prop:map-candidate-selection}}
\mapcandidateselection*

\begin{proof}
Adding the candidate $(\xtest,c)$ to the calibration set adds $\log p_\theta(c\mid\xtest)$ to the log posterior. Thus the augmented MAP $\hat\theta_{n,c}$ globally maximizes $\gL(\theta)+\log p_\theta(c\mid\xtest)$.
With $\gL(\theta')\geq\gL(\hat\theta_n)-r$, and
    $\log p_{\theta'}(c\mid\xtest)\geq\log v${,} we have:
\begin{align*}
    \gL(\hat\theta_{n,c})+\log p_{\hat\theta_{n,c}}(c\mid\xtest)
    &\geq \gL(\theta')+\log p_{\theta'}(c\mid\xtest)\\
    &\geq \gL(\hat\theta_n)-r+\log v.
\end{align*}
Since $\hat\theta_n$ maximizes the calibration log posterior, $\gL(\hat\theta_{n,c})\leq\gL(\hat\theta_n)$. Subtracting $\gL(\hat\theta_{n,c})$ from the preceding inequality therefore yields $\log p_{\hat\theta_{n,c}}(c\mid\xtest)\geq\log v-r$, or $p_{\hat\theta_{n,c}}(c\mid\xtest)\geq e^{-r}v$. This bound holds whether or not $\log(v/u)>r$. If that condition holds, then $e^{-r}v>u$, whereas every model in $\gB_0$ assigns $c$ probability at most $u$. Hence $\hat\theta_{n,c}\notin\gB_0$.
\end{proof}

\section{Supplementary Results, and Experiments}
\label{app:supplementary-results}

\begin{figure}
    \centering
    \includegraphics[scale=0.9]{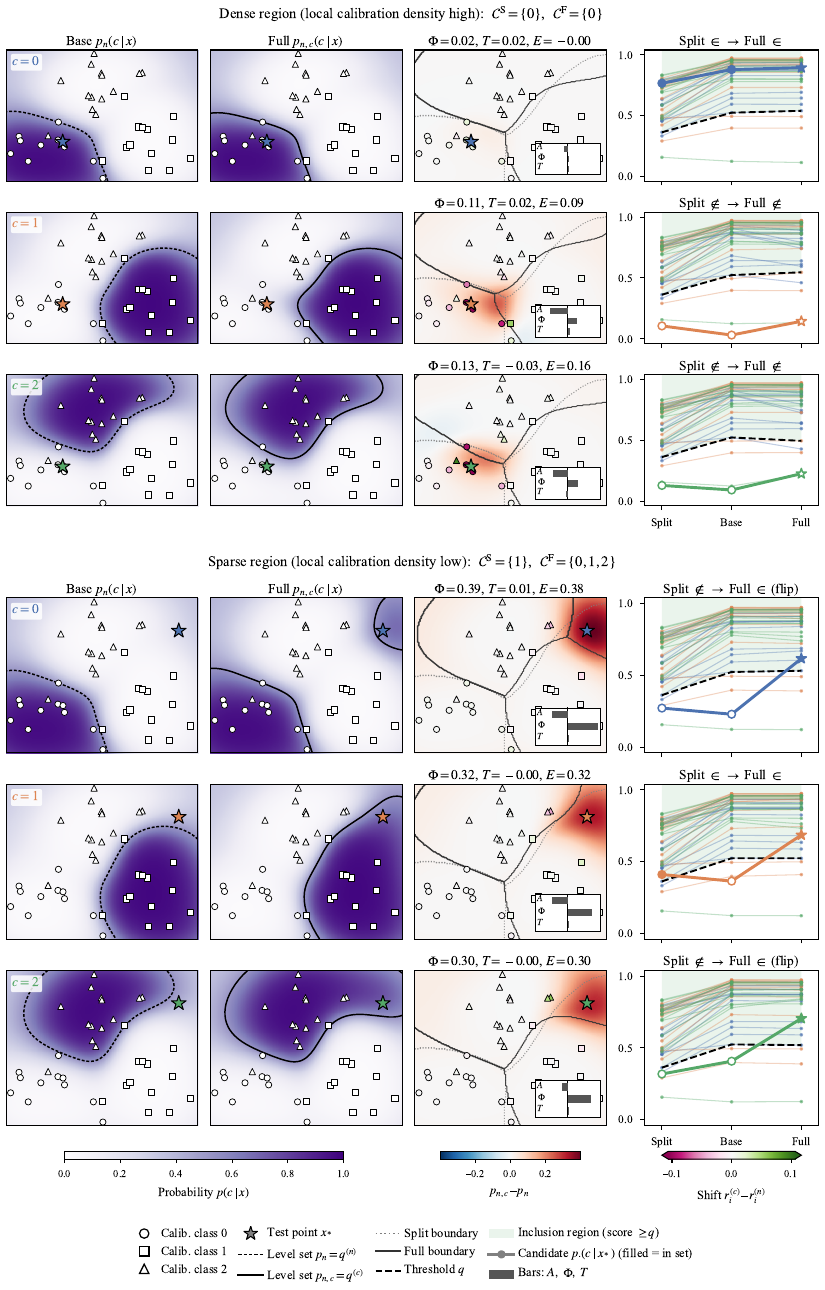}
    \caption{Complete anatomy of the split-to-full margin.}
    \label{fig:score-anatomy-appendix}
\end{figure}

\subsection{Synthetic Experiments}
\label{app:synthetic-experiments}

\subsubsection{Anatomy of the Split-to-Full CP}
\label{app:score-anatomy}

\begin{figure}[t]
    \centering
    \includegraphics{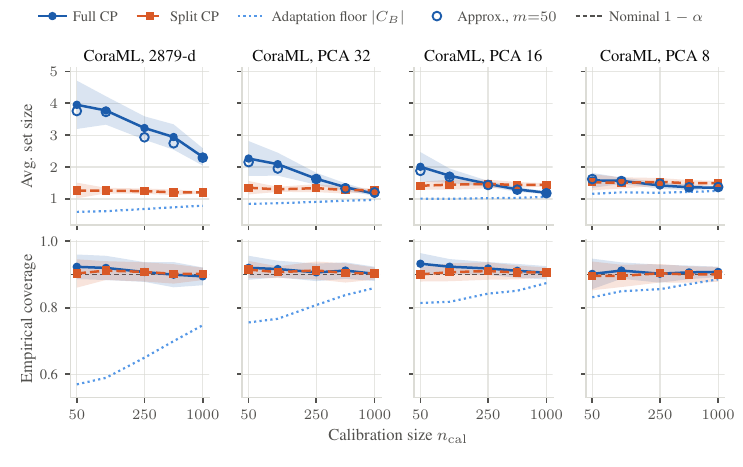}
    \caption{Full (blue, solid) vs.\ split (orange, dashed) CP on \texttt{CoraML} (transductive node-classification, GCN) over the calibration set size $n_{\mathrm{cal}}$, with the $2879$-d bag-of-words features [Left], and their PCA reduction to $K=32$, $16$, and $8$ [Middle Left to Right]. [Top] Average set size against the adaptation floor $|\gC_B|$ (dotted). Hollow markers are the approximate full CP with $m=50$. [Bottom] Empirical coverage of the same sets against the nominal $1-\alpha=0.9$ (black, dashed). Bands are $\pm1$ standard deviation across calibration splits.}
    \label{fig:full-vs-split-coraml}
\end{figure}

\begin{figure}[t]
    \centering
    \includegraphics{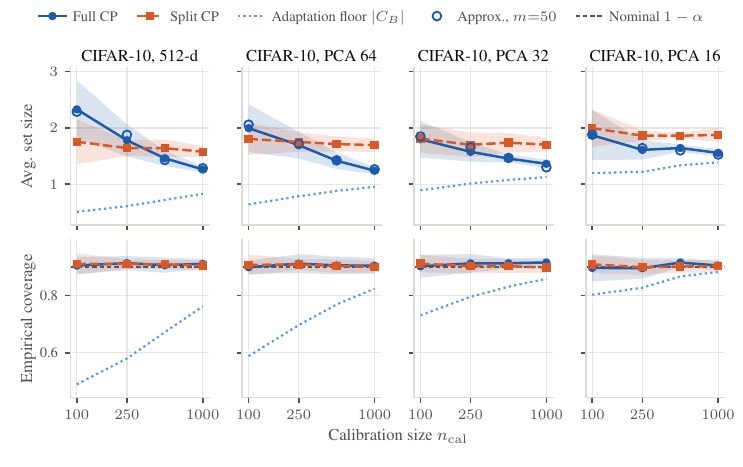}
    \caption{Full (blue, solid) vs.\ split (orange, dashed) CP on \texttt{CIFAR-10} (two-layer MLP over a frozen ResNet-18) over the calibration set size $n_{\mathrm{cal}}$, with the $512$-d embedding [Left], and its PCA reduction to $K=64$, $32$, and $16$ [Middle Left to Right]. [Top] Average set size against the adaptation floor $|\gC_B|$ (dotted). Hollow markers are the approximate full CP with $m=50$. [Bottom] Empirical coverage of the same sets against the nominal $1-\alpha=0.9$ (black, dashed). Bands are $\pm1$ standard deviation across calibration splits.}
    \label{fig:full-vs-split-cifar10}
\end{figure}

This experiment replicates the transition from split to full conformal prediction in a two-dimensional setup where every quantity of \autoref{eq:margin-decomposition} is directly visible. \autoref{fig:score-anatomy} in \autoref{sec:motivating-experiment} shows a compact presentation of the result (one correct, and one wrong candidate per test point). \autoref{fig:score-anatomy-appendix} reports the complete anatomy for every candidate label at both test points, in addition to the score function at each stage.

\parbold{Setup.} We sample three classes from overlapping Gaussians (means on a triangle with pairwise distance $4$, standard deviation $1.5$). The model is a multinomial logistic regression over fixed RBF features, 
so that every refit is an exact convex MAP problem. The pretrained model $\vtheta_0$ is trained on $3$ points per class (with a prior centered at zero), and every fine-tuning minimizes the cross-entropy plus the $L_2$ pull $\frac{1}{2\tau^2}\|\vtheta-\vtheta_0\|^2$ with $\tau^2=1$; i.e. the MAP objective under the prior $\gN(\vtheta_0,\tau^2\mI)$. The calibration set has $n=42$ points, the score is the predictive probability, and $1-\alpha=0.9$. We fine-tune once on $\dcal$ (the base model $p_n$), and once per candidate on $\daugmented$ (the full models $p_{n,c}$). We {examine} two test points that differ only in their local calibration density: one is placed inside the class-$0$ cluster, the other away from every calibration point.

\parbold{Description for \autoref{fig:score-anatomy-appendix}.} Rows iterate over the candidate labels $c\in\{0,1,2\}$. {There are two} blocks, one for the test point {in the dense} area (top), and the other (bottom) the sparse one. Columns from left to right are \begin{enumerate*}[label=(\roman*)]
    \item the base predictive field $p_n(c\mid\vx)$ together with its level set at the base threshold $q_\vtheta^{(n)}$ (dashed): the region inside the {contour} is where the label $c$ would enter the prediction set under the base model.
    \item The same field after fine-tuning on the pseudo-label, $p_{n,c}(c\mid\vx)$, with the level set at $q_\vtheta^{(n,c)}$ (solid).
    \item Deformation $p_{n,c}(c\mid\vx)-p_n(c\mid\vx)$. The calibration markers are colored by the shift of their own score, $r_i^{(c)}-r_i^{(n)}=C_{i,c}/w_{n,c}$ from \autoref{thm:full-split-decomposition}. The figure also shows the quantities $A_{n,c}$, $\Phi_{n,c}$, and $T_{n,c}$ on one scale shared across all plots.
    \item Shows every score through the three regimes (split, base, full): thin lines are the calibration scores, the dashed line is the threshold, the shaded area is where the label is included in the set, and the bold line is the candidate itself, with filled markers indicating membership at that stage.
\end{enumerate*}

{As expected}, at the dense region the self-pull is resisted by the surrounding calibration points (and the training) with $\Phi_{n,c}\leq0.13$ for the wrong candidates. In other words, the deformation is small, and sharply localized. The threshold also barely moves, and the two {level sets remain} relatively steady. {Conversely}, at the sparse test point where there is much less resistance from other points, {fine-tuning} on \emph{any} pseudo-label bends the model freely ($\Phi_{n,c}\approx0.3$--$0.4$ for every candidate), no calibration score reacts ($C_{i,c}\approx0$, hence $T_{n,c}\approx0$), and as a result, every candidate -- the wrong ones included -- flips into the full set. The {takeaway} is that the full CP loses exactly where the formability $\Phi_{n,c}=V_{n,c}/w_{n,c}$ is large.

\subsubsection{Aleatoric against Epistemic Uncertainty}
\label{app:epistemic-sweeps}

\autoref{fig:epistemic-sweeps} in \autoref{sec:motivating-experiment} reads two uncertainty measures on the setup of \autoref{app:score-anatomy} while the aleatoric part of the uncertainty is moved on purpose. The generating law is known, so the aleatoric part is known in closed form: with $\vmu_k$ the class means (at pairwise distance $d$), and the common standard deviation $1.5$, the Bayes posterior at $\vx$ is $\eta_k(\vx)\propto\exp\bigl(-\|\vx-\vmu_k\|^2/(2\cdot1.5^2)\bigr)$, and we report $1-\frac{K}{K-1}\mathrm{TV}\bigl(\eta(\vx),\tfrac1K\bigr)$, with $\mathrm{TV}$ the total variation distance to the uniform law over the $K=3$ classes. It is $1$ where the label is a coin flip among the three classes, $1/2$ where it is a coin flip between two of them, $0$ where it is determined, and it depends neither on the model nor on the sample. The two measures are read from the model fine-tuned on $\dcal$: the normalized softmax entropy $H(p_n(\cdot\mid\xtest))/\log K$, and the formability variance $\sum_c\hat V_{n,c}=\sum_c\hat w_{n,c}\hat\Phi_{n,c}$ of \autoref{eq:total-variance-from-formability}, with $\hat w_{n,c}=p_n(c\mid\xtest)$, and $\hat\Phi_{n,c}=p_{n,c}(c\mid\xtest)-p_n(c\mid\xtest)$ from one MAP refit per pseudo-label, exactly as in \autoref{app:score-anatomy}.

\parbold{Setup.} Model, pretraining ($3$ points per class), calibration ($n=42$), prior, and $\alpha$ are those of \autoref{app:score-anatomy}. The grid of RBF centers is the one of that experiment extended by the same spacing (to $15\times14$ centers), so that it covers the wider configurations, and the longer paths below. The data are drawn once as standard normal offsets $\vz_i$, and placed as $\vx_i=\vmu_{y_i}(d)+1.5\,\vz_i$, so that the sample moves continuously with $d$. The means $\vmu_k(d)$ are the means of \autoref{app:score-anatomy} scaled about their centroid $\bar\vmu$ to pairwise distance $d$ ($d=4$ is the configuration of \autoref{fig:score-anatomy}), and the pretrained model, and $p_n$ are refitted at every $d$. In \autoref{fig:epistemic-sweeps}(b), $d$ runs from $8$ to $0$, and $\xtest=\bar\vmu$. In (c), and (d), $d=4$, and $\xtest(t)=\bar\vmu+t\,\|\vmu_0-\bar\vmu\|\,\vu$ for $t\in[0,3.5]$ along a unit direction $\vu$: in (c) towards the center of class $0$, which the path reaches at $t=1$; in (d) towards the midpoint $\bar\vmu_{12}$ of classes $1$, and $2$, which it reaches at $t=0.5$. The path of (d) lies, up to $0.01$, on the perpendicular bisector of $\vmu_1$, and $\vmu_2$, so the Bayes posterior stays a two-way coin flip between the two classes along the rest of it. Every curve is the mean over $5$ draws of the data, and the band is their range. The aleatoric reference does not depend on the draw.

\parbold{Reading.} In (b), the label at $\xtest$ is a three-way coin flip at every $d$ (the centroid is, up to $0.02$, equidistant from the three means), so the aleatoric reference is at least $0.976$ throughout. The softmax entropy reads this: its mean stays between $0.82$ and $0.92$ (the range over the draws is $0.59$--$0.98$). The labeled points, however, pile up around $\xtest$ as the classes converge -- the Gaussian kernel density of $\dcal$ at $\xtest$ (bandwidth $1.1$) grows from $0.002$ to $0.046$ -- and the formability variance reads this: it falls in every draw, from $0.276$ to $0.061$ on average ($\sum_c\hat\Phi_{n,c}$ from $0.94$ to $0.20$). The two measures move independently: the entropy is flat, while the formability variance falls by a factor of $4.5$. In (c), the aleatoric reference falls from $0.994$ to $0.079$ at $\vmu_0$, and to $0.002$ at $t=2$. Both measures fall towards $\vmu_0$, the entropy from $0.92$ to $0.32$, and the formability variance from $0.166$ to $0.024$, as they must: with $p_n(0\mid\xtest)$ near one every plausible model agrees, and $V_{n,c}\leq w_{n,c}(1-w_{n,c})$ (\autoref{lem:formability-variance}) leaves no room for disagreement. Both reach their minimum at $t=1.4$ ($0.22$, and $0.012$), where the labeled points start to thin out (the density is $0.018$ there, $0.009$ at $t=2$, $0.003$ at $t=2.5$, and below $10^{-3}$ from $t=3$ on), and then both rise: the formability variance to $0.37$, since a model that no data resist is free to deform towards any label ($\sum_c\hat\Phi_{n,c}=1.1$; every refit lifts its pseudo-label to about $0.7$), and the entropy to $0.99$, since the RBF logits decay to the bias away from the data, and the softmax returns to the near-uniform bias. The entropy therefore reads the same value, its maximum, at the centroid and at the far end of the path, where the aleatoric reference is $0.994$, and $0$: the first reading is aleatoric, the second is entirely epistemic, and the entropy cannot tell the two apart. In (d), the aleatoric reference falls from $0.994$ to $0.62$ at the midpoint $\bar\vmu_{12}$, and then settles at $0.50$: the label remains a coin flip between classes $1$, and $2$ along the rest of the path. The entropy reads $0.74$ at the midpoint, its minimum $0.69$ at $t=0.9$, and then climbs back to $0.99$. The formability variance reads $0.12$ at the midpoint, and rises to $0.37$ as the density falls from $0.014$ to $0$. Between $t=1$ and $t=3.5$ the aleatoric part is constant, so the rise of the entropy from $0.69$ to $0.99$ is entirely epistemic; the entropy does not say so, and the formability variance does. Wherever the entropy is high, then, the formability variance separates the two cases: it is small where the entropy is high because the labels are mixed and the labeled points are many (the centroid at small $d$, the neighborhood of $\bar\vmu_{12}$), and large where the entropy is high because nothing constrains the model (the centroid at $d=8$, the far ends of both paths).

\subsubsection{Exact Evaluation on Bayesian Linear Regression}
\label{app:blr-check}
The margin decomposition for full - split CP in \autoref{sec:theory} is general. {However, using} the Bayesian learning regime we can show that $\Phi_{n, c}$ is a function of epistemic uncertainty. {Such a setup} is unrealistic for large-scale deep learning, for which we discuss the same concept under deterministic fine-tuning in \autoref{sec:map-refits}. Meanwhile, the Bayesian setup is exactly realizable in a small Gaussian-linear model, where every quantity of \autoref{sec:theory} can be evaluated in closed form. 
The synthetic dataset is drawn from $f^\star(x)=\sin(1.5x)+0.3x$ plus noise: $40$ pretraining points uniform on $[-4,-0.5]$, {and the calibration set is drawn} from the mixture $x \sim \tfrac34\gN(0,0.5^2)+\tfrac14\,\mathrm{U}[-4,4]$, therefore its density is high around $x=0$ and low elsewhere ($n=60$).
The model is $y=\vtheta^\top\vpsi(x)+\varepsilon$ with $\varepsilon\sim\gN(0,\sigma^2)$, $\sigma=0.5$, and $\vpsi(x)\in\R^{13}$ a bias plus $12$ Gaussian bumps of width $0.7$ centered on a grid over $[-4,4]$. The score is the likelihood $p_\vtheta(y\mid x)=\gN(y;\vtheta^\top\vpsi(x),\sigma^2)$, which takes values in $[0,\,1/(\sigma\sqrt{2\pi})]=[0,0.8]$. The bound $V_{n,c}\leq\Phi_{n,c}(1-\Phi_{n,c})$ of \autoref{lem:formability-variance} therefore applies to it as it does to a class probability, and the candidates $c$ range over a grid on $[-3,3]$.  The pretrained model $\vtheta_0$ is trained with prior $\gN(\vzero,\mI)$, and $1-\alpha=0.9$.

\parbold{Closed forms.} For $\vtheta\sim\gN(\vmu,\mS)$ write $m(x):=\vmu^\top\vpsi(x)$ and $s^2(x):=\vpsi(x)^\top\mS\vpsi(x)$. The calibration posterior is $\pi_n=\gN(\vmu_n,\mS_n)$ with $\mS_n^{-1}=\mI+\sigma^{-2}\sum_{i=1}^n\vpsi(x_i)\vpsi(x_i)^\top$ and $\vmu_n=\mS_n\bigl(\vtheta_0+\sigma^{-2}\sum_{i=1}^ny_i\vpsi(x_i)\bigr)$, and $\pi_{n,c}$ is the same update with the pair $(\xtest,c)$ added. This update is the refit that produces $m_{n,c}$, $\Phi_{n,c}$, and $T_{n,c}$ in \autoref{fig:section3}. The moments that enter \autoref{lem:formability-variance} and \autoref{thm:full-split-decomposition} are, with $\mPsi:=[\vpsi(x),\vpsi(x')]$,
\begin{align*}
    \E_{\vtheta}[p_\vtheta(y\mid x)]&=\gN\bigl(y;\,m(x),\,\sigma^2+s^2(x)\bigr),\\
    \E_{\vtheta}[p_\vtheta(y\mid x)^2]&=\frac{1}{2\sigma\sqrt\pi}\,\gN\bigl(y;\,m(x),\,\sigma^2/2+s^2(x)\bigr),\\
    \E_{\vtheta}[p_\vtheta(y\mid x)\,p_\vtheta(y'\mid x')]&=\gN\bigl((y,y');\,(m(x),m(x')),\,\sigma^2\mI+\mPsi^\top\mS\mPsi\bigr),
\end{align*}
{which are used} to compute the scores $r_i$ and the candidate mass $w_{n,c}$, the variances $V_i$ and $V_{n,c}$, and the forming kernel $C_{i,c}$ of \autoref{eq:formability-kernel}. Each identity of \autoref{sec:theory} is then compared against the quantity it predicts, computed independently from the refit $\pi_{n,c}$. 

\parbold{Conclusion.} Our sanity check (which is also shown in \autoref{fig:section3}) succeeds as over all $481\times3$ candidate--test pairs, the two independent computations are close up to the numerical instability; i.e. $\max_{i,c}|r_i^{(c)}-r_i^{(n)}-C_{i,c}/w_{n,c}|=7\cdot10^{-15}$, $\max_c|V_{n,c}-w_{n,c}\Phi_{n,c}|=3\cdot10^{-16}$, $|C_{i,c}|\leq\sqrt{V_iV_{n,c}}$ throughout, and $m_{n,c}$ lies inside the band of \autoref{eq:bayesian-margin-decomposition} with a smallest slack of $6.5\cdot10^{-3}$ (inside the kernel band of \autoref{eq:threshold-drift-sandwich}, $2.4\cdot10^{-4}$). The formability is the dominant term wherever the neighborhood is sparse ($\max_c\Phi_{n,c}=0.24$ against $\max_c|T_{n,c}|=0.03$ at the sparse test point), which is why the full set inflates there and not at the dense point, where $\max_c\Phi_{n,c}=0.02$. 

\subsubsection{The Frozen Image Backbones}
\label{app:backbones}

Every image experiment in this paper trains a small head on top of a frozen \texttt{ImageNet}{ \citep{deng2009imagenet}}-pretrained backbone, and the feature space from which the {formability} is defined is the output of this backbone (this {especially} matters for CIFAR-10 experiments with feature reduction). The \texttt{ImageNet} output layer is removed, the remaining network embeds each image once, and only the head is trained, or refitted. The backbone itself is never updated, which helps us to reduce the cost of per-candidate refit from the entire model to only the head. \autoref{tab:backbones} {lists} the specifications{ of the ResNet \citep{he2016deep}, ViT \citep{dosovitskiy2021image}, and ConvNeXt \citep{liu2022convnet} backbones. ViT-H/14 uses the SWAG weights of \citet{singh2022revisiting}}.

\begin{figure}[t]
    \centering
    \includegraphics{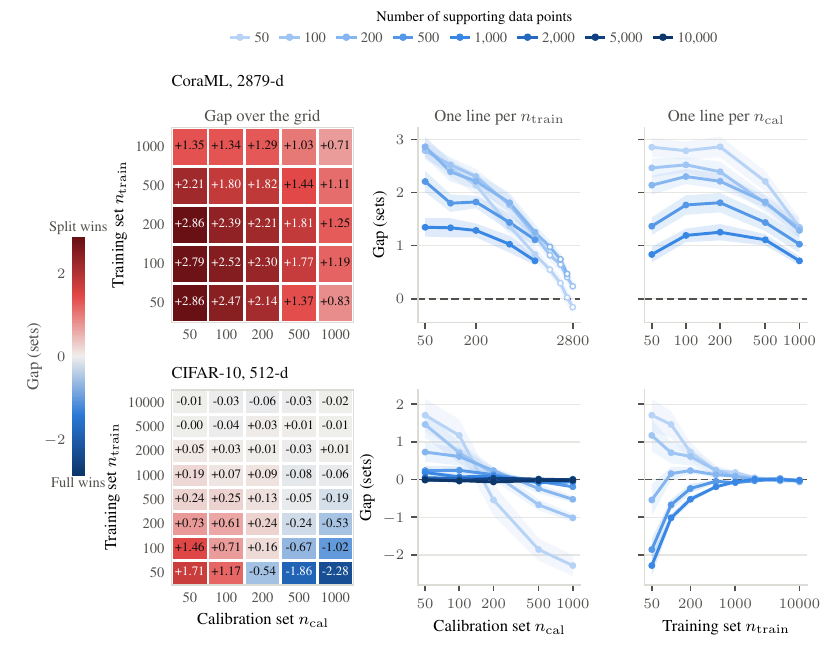}
    \caption{The gap $|\gC^F|-|\gC^S|$ in average set size over the training set size $n_{\mathrm{train}}$, and the calibration set size $n_{\mathrm{cal}}$ (positive: split CP is smaller, negative: full CP is smaller), on \texttt{CoraML} with $2879$-d features [Top], and \texttt{CIFAR-10} with the full $512$-d embedding [Bottom]. [Left] The gap over the grid. [Middle] The gap against $n_{\mathrm{cal}}$, one line per $n_{\mathrm{train}}$. Hollow markers are the \texttt{CoraML} extension arm up to $n_{\mathrm{cal}}=2800$. [Right] The gap against $n_{\mathrm{train}}$, one line per $n_{\mathrm{cal}}$. In the line panels, the shade encodes the size of the other set (light to dark = small to large), and bands are $\pm1$ standard error of the paired difference across calibration splits. $1-\alpha=0.9$.}
    \label{fig:full-vs-split-train-calib}
\end{figure}

\begin{table}[t]
    \centering
    \caption{The frozen \texttt{ImageNet}-pretrained backbones over which the image experiments are run. \emph{Weights} is the \texttt{torchvision} checkpoint. \emph{Original} counts the parameters of the classifier as loaded, and \emph{fixed} the ones that remain once its \texttt{ImageNet} output layer is removed, which is the form in which we load it. $d$ is the dimension of the embedding the head sees, and it is also the dimension the PCA reduction of \autoref{app:full-vs-split-pca} starts from. The top-1 accuracy is the reference number reported for the same checkpoint on \texttt{ImageNet}-1k.}
    \label{tab:backbones}
    \setlength{\tabcolsep}{5pt}
    \begin{tabular}{llrrrr}
        \toprule
        & & \multicolumn{2}{c}{Parameters} & Feature & \texttt{ImageNet}-1k \\
        \cmidrule(lr){3-4}
        Backbone & Weights & Original & Fixed & dim.\ $d$ & top-1 (\%) \\
        \midrule
        ResNet-18   & \texttt{IMAGENET1K\_V1}            & $11.7$M  & $11.2$M  & $512$  & $69.8$ \\
        ResNet-50   & \texttt{IMAGENET1K\_V2}            & $25.6$M  & $23.5$M  & $2048$ & $80.9$ \\
        ViT-B/16    & \texttt{IMAGENET1K\_V1}            & $86.6$M  & $85.8$M  & $768$  & $81.1$ \\
        ConvNeXt-B  & \texttt{IMAGENET1K\_V1}            & $88.6$M  & $87.6$M  & $1024$ & $84.1$ \\
        ViT-H/14    & \texttt{IMAGENET1K\_SWAG\_E2E\_V1} & $633.5$M & $632.2$M & $1280$ & $88.6$ \\
        \bottomrule
    \end{tabular}
\end{table}

\subsubsection{Full vs.\ Split CP over the Feature Dimension}
\label{app:full-vs-split-pca}

\parbold{Setup.} We use a two-layer GCN \citep{kipf2017semi} over the $2879$-dimensional bag-of-words node features of \texttt{CoraML}, and a two-layer MLP over the $512$-dimensional embedding of a frozen \texttt{ImageNet}-pretrained ResNet-18 for \texttt{CIFAR-10}. Full CP trains the model from initialization on the training set, the calibration set, and the pseudo-labeled test point. We reduce the features by PCA to $K\in\{32,16,8\}$ for \texttt{CoraML}, and $K\in\{64,32,16\}$ for \texttt{CIFAR-10}, and sweep $n_{\mathrm{cal}}$, and $n_\mathrm{tr}$. The score is the predictive probability of \autoref{sec:background}, and $1-\alpha=0.9$. The adaptation floor $|\gC_B|$ is the size of the posterior conformal set $\gC_{\vtheta_n}(\xtest)$ i.e. split CP on the model trained on the training, and the calibration set together, which needs a single fit per calibration set, and no per-candidate refit.

\parbold{Results.} In \autoref{fig:full-vs-split-coraml}, and \autoref{fig:full-vs-split-cifar10} the plots go from the full feature space to the smallest $K$. The top row reports the average set size against $n_{\mathrm{cal}}$, where the hollow markers are the approximate full CP with $m=50$, and the bottom row reports the empirical coverage of the same sets against the nominal $1-\alpha$. Full, and split CP are at the nominal level in every panel which shows the sanity of the result. The adaptation floor is not (as expected): it under-covers by a wide margin in the full feature space ($\approx0.57$ on \texttt{CoraML} at $n_{\mathrm{cal}}=50$, and $\approx0.49$ on \texttt{CIFAR-10} at $n_{\mathrm{cal}}=100$), and it approaches $1-\alpha$ as $n_{\mathrm{cal}}$ grows, or $K$ shrinks. This is expected, since $\gC_B$ scores the calibration points it was trained on, and the amount of this optimism is again a formability -- now read at the calibration points. 

Similarly in set size, we observe that by increasing the calibration points, the difference between the two approaches decreases, until full CP crosses the split CP. Notably this effect is more visible in \texttt{CIFAR-10} than in \texttt{CoraML}. Mainly this is since the training size of the \texttt{CIFAR-10} is larger than that of \texttt{CoraML}. We also observe that reducing the feature dimension closes the performance gap.

An interesting side observation is that the question of whether the full CP is more efficient than the split CP can be answered without the need to compute the full CP. One can simply estimate it with a small sample rate ($m=50$) by computing a leave-one-out refit on the calibration set, and that point alone. This is a relatively cheap operation to check the performance gap, without the need for expensive computation. The method turns out to be a good approximation as shown in both figures (the hollow markers are the approximate full CP with $m=50$).

We report the extended experiment on the gap $|\gC^F(\xtest)|-|\gC^S(\xtest)|$ over a grid of $n_{\mathrm{train}}$, and $n_{\mathrm{cal}}$, on \texttt{CoraML} ($2879$-d), and on \texttt{CIFAR-10} (the $512$-d embedding). The results are shown in \autoref{fig:full-vs-split-train-calib}. This is also same as our expectation. Another takeaway from this result is that the training set size also decreases the performance, and as the training size increases, the effect of the calibration set size becomes less significant. 

\subsubsection{Adaptive Prediction Sets}
\label{app:train-vs-calib-aps}

In this work we focus on threshold prediction sets as the score function \citep{sadinle2019least}: it directly sets the conformity score to the predictive probability of the candidate label; i.e. $s(\vx,y)=p_\vtheta(y\mid\vx)$. That object is very easy to work with and it directly connects to the decomposition of \autoref{eq:margin-decomposition} to the updates of the predictive distribution after refits. However, conformal prediction is agnostic to the choice of score, and the same margin-decomposition can apply on other score functions as well. It is an important question to ask what interpretation does the full-split gap (and in particular the margin gap) has under that different score. 

We repeat the experiment shown in \autoref{fig:full-vs-split-cifar10}, this time under the adaptive prediction sets (APS) score \citep{romano2020classification}. Result is shown in \autoref{fig:full-vs-split-train-calib-aps}. This ablation does not need to train a new model since we have already cached the refits from the previous experiment. 

\begin{figure}[t]
    \centering
    \includegraphics{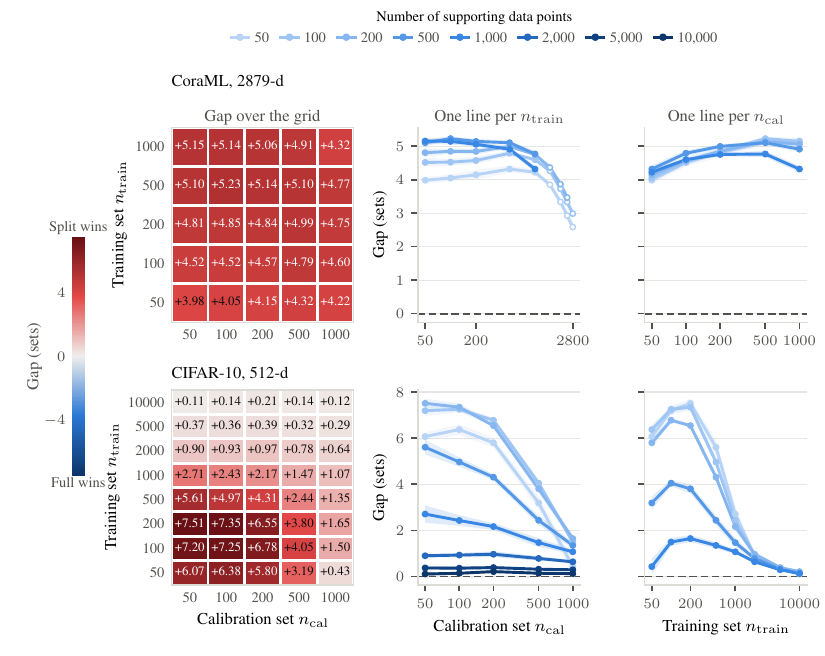}
    \caption{The grid of \autoref{fig:full-vs-split-train-calib} read under the randomized APS score \citep{romano2020classification}, off the same refits: the gap $|\gC^F|-|\gC^S|$ in average set size over the training set size $n_{\mathrm{train}}$, and the calibration set size $n_{\mathrm{cal}}$ (positive: split CP is smaller), on \texttt{CoraML} with $2879$-d features [Top], and \texttt{CIFAR-10} with the full $512$-d embedding [Bottom]. [Left] The gap over the grid. [Middle] The gap against $n_{\mathrm{cal}}$, one line per $n_{\mathrm{train}}$. Hollow markers are the \texttt{CoraML} extension arm up to $n_{\mathrm{cal}}=2800$. [Right] The gap against $n_{\mathrm{train}}$, one line per $n_{\mathrm{cal}}$. In the line panels, the shade encodes the size of the other set (light to dark = small to large), and bands are $\pm1$ standard error of the paired difference across calibration splits. The color scale is centered at zero as in \autoref{fig:full-vs-split-train-calib}, but it spans a range about two and a half times wider. No cell of either grid is negative. $1-\alpha=0.9$.}
    \label{fig:full-vs-split-train-calib-aps}
\end{figure}

APS applies the same conformal prediction as shown in \autoref{sec:background}, but instead on the following score function:
\begin{align*}
    s(\vx,y)=1-\sum_{y':\,p_\vtheta(y'\mid\vx)>p_\vtheta(y\mid\vx)}p_\vtheta(y'\mid\vx)-u\,p_\vtheta(y\mid\vx),
\end{align*}
and we use the randomized variant, with one $u\sim\mathrm{Unif}(0,1)$ drawn per (point, rank) once per split, and reused by every refit. The score therefore remains a deterministic function of the predictive distribution, and the exchangeability argument is untouched. Everything else -- the grid, the test points, the models, and $1-\alpha=0.9$ -- is as in \autoref{fig:full-vs-split-train-calib}.

\parbold{Result.} Split CP is smaller in every cell of both grids, and by a wide margin, and by increasing the training points or calibration points the gap converges to zero. APS can be read as a fractional knapsack problem over the probabilty distribution: take the smallest number of labels (all with equal weight), until $1 - \beta$ cumulative probability is captured. Then conformalization (computing the quantile) is an answer to: what cumulative $1 - \beta$ in model's predictive probability can capture the $1 - \alpha$ probability in the true distribution. The gap is therefore a measure of how much the model's predictive distribution is misaligned with the true distribution, and the gap converges to zero as the model sees more labels and becomes better aligned with the true distribution. However deriving the same formability decomposition as in \autoref{sec:theory} is not straightforward, specially since the score involves the ranking of the labels, and the refit can change this ranking. We leave this as an open question for future work.

\subsubsection{Breakdown of Uncertainty Quantification}
\label{app:uncertainty-breakdown}

In \autoref{fig:leaderboard-grouped-all} we report the detailed comparison against other uncertainty measures. Formability is close to the best methods on graphs and images, while on language softmax entropy performs better. The full--split set-size gap performs poorly across all three modalities, as it quantizes the continuous formability signal into a binary signal.

\begin{figure}[t]
  \centering
  \includegraphics{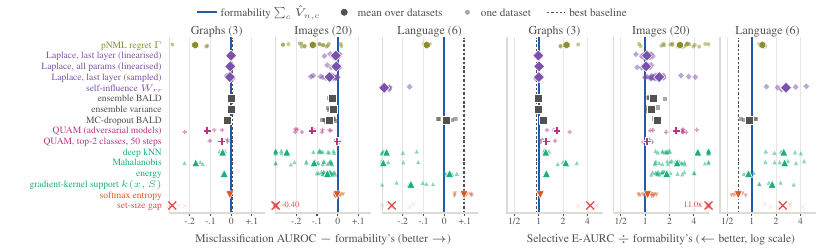}
  \caption{
    Formability against families of uncertainty {measures, for in-distribution} data, on graphs, images and language. For each modality the relative performance is gathered from different datasets and models. 
  }
  \label{fig:leaderboard-grouped-all}
\end{figure}

\autoref{fig:backbone-ladder-main} extends the image comparison of \autoref{fig:leaderboard-grouped}, and \autoref{fig:pnml-regret-main} across different backbones and distribution shifts. Formability stays close to the best uncertainty measures for ranking errors on clean images, while for shift detection the result depends on the type of shift and is weaker on Gaussian noise.

\begin{figure}[t]
  \centering
  \includegraphics[width=\linewidth]{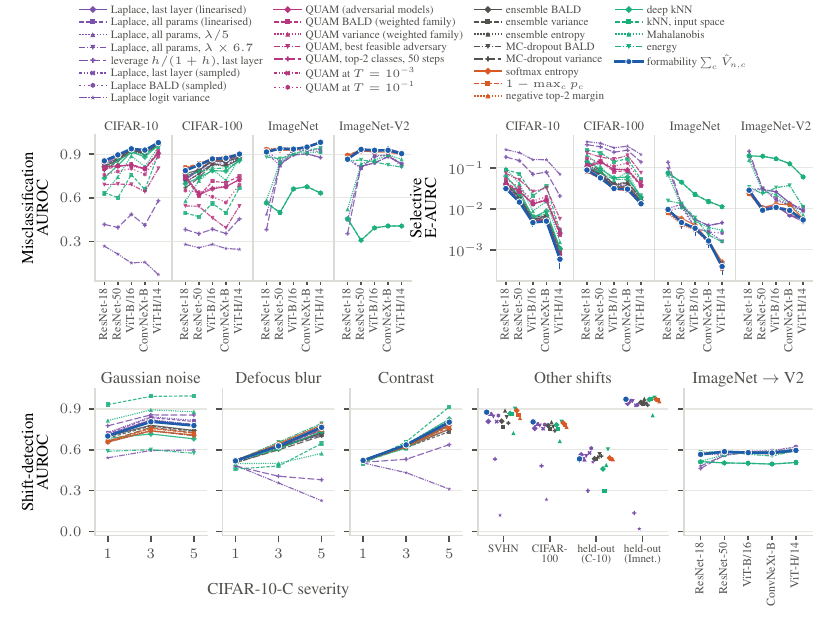}
  \caption{Formability against uncertainty measures across frozen image backbones and distribution shifts. [Top] Misclassification AUROC (higher is better) and selective E-AURC (lower is better, log scale) on clean \texttt{CIFAR-10}, \texttt{CIFAR-100}, \texttt{ImageNet}, and \texttt{ImageNet-V2}, with five backbones in each block. [Bottom] Shift-detection AUROC for \texttt{CIFAR-10-C} at three severities, Street View House Numbers (\texttt{SVHN}), \texttt{CIFAR-100}, held-out classes, and \texttt{ImageNet-V2} against \texttt{ImageNet} validation images. The blue line is $\sum_c\hat V_{n,c}$. Bars are $\pm1$ standard error over splits. Other colors distinguish Laplace, ensemble and MC-dropout, QUAM, feature-based, and softmax measures.}
  \label{fig:backbone-ladder-main}
\end{figure}

\autoref{fig:aleatoric-bins-error-ranking} further extends the image comparison of \autoref{fig:leaderboard-grouped} by grouping \texttt{CIFAR-10} images according to human-label disagreement (which can be considered a proxy for aleatoric uncertainty). Formability remains close to the best methods consistently within these groups, while ranking the errors generally becomes harder as the disagreement increases.

\begin{figure}[t]
  \centering
  \includegraphics[width=\linewidth]{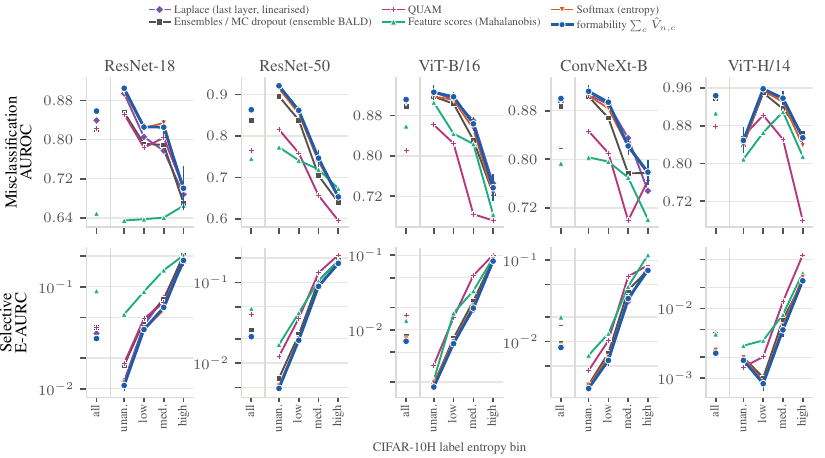}
  \caption{Error ranking within bins of human-label disagreement on \texttt{CIFAR-10H}. The $1000$ \texttt{CIFAR-10} test images are grouped by the entropy $H$ of about $51$ human labels per image: unanimous ($H=0$), low ($0<H<0.2$), medium ($0.2\le H<0.5$), and high ($H\ge0.5$). ``All'' is the pooled result. Each column is a frozen backbone. [Top] Misclassification AUROC (higher is better). [Bottom] Selective E-AURC (lower is better, log scale). Blue: $\sum_c\hat V_{n,c}$ with $\pm1$ standard error over splits. Other lines: last-layer Laplace, ensemble BALD, QUAM, Mahalanobis, and softmax entropy.}
  \label{fig:aleatoric-bins-error-ranking}
\end{figure}

\parbold{Ablation on training set size.} In addition we also stress-test the measures against the number of labels the model has seen during training.
\autoref{fig:size-ladder} shows this comparison across various training budgets.

The models are the ones of \autoref{app:full-vs-split-pca}. We vary the number of labels the model has seen, $N=n_{\mathrm{train}}+n_{\mathrm{cal}}$, from $50$ to $1600$ on \texttt{CoraML}, and from $50$ to $3000$ on \texttt{CIFAR-10} -- with the calibration set held at $n_{\mathrm{cal}}=20$. Note that here our main argument is not to show a powerful conformal prediction, but to show the quality of formability in capturing epistemic uncertainty. Therefore the small calibration size is justified for computational efficiency. Columns are the two datasets, with the accuracy of the base model at each rung along the top. [Top] The misclassification AUROC (higher is better; the dotted rule is chance). [Bottom] The selective E-AURC (lower is better). In all training budgets, formability remains among the best methods.

\begin{figure}[t]
    \centering
    \includegraphics{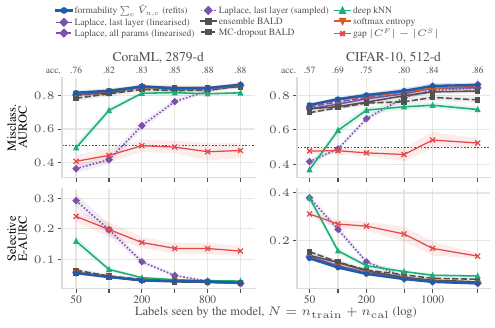}
    \caption{The epistemic measures of \autoref{sec:experiments} against the number of labels the model has seen, $N=n_{\mathrm{train}}+n_{\mathrm{cal}}$ (log), at the end of training (epoch $300$), on \texttt{CoraML} with the $2879$-d node features [Left], and \texttt{CIFAR-10} with the $512$-d embedding [Right]. The accuracy of the base model at each rung is printed along the top. [Top] Misclassification AUROC (higher is better; dotted: chance). [Bottom] Selective E-AURC (lower is better). Every measure ranks the same base model on the same $200$ evaluation points of a split, with $n_{\mathrm{cal}}=20$ held fixed. Blue circles: $\sum_c\hat V_{n,c}$ from the refits. Purple diamonds: the linearized last-layer Laplace (solid), the all-parameter empirical-Fisher form (dashed), and the sampled last-layer Laplace (dotted). Grey squares: ensemble BALD (solid), and MC-dropout BALD (dashed). Green triangles: deep kNN. Orange: softmax entropy. Red crosses: the set-size gap $|\gC^F|-|\gC^S|$. Lines are means over $8$ calibration splits, bands $\pm1$ standard error.}
    \label{fig:size-ladder}
\end{figure}

\parbold{Cost against quality.}
The refits are the most expensive reading of the epistemic uncertainty by construction: one fine-tuning per pseudo-label, per test point. We evaluate the cost, and (AUROC, E-AURC) on different measures for five datasets: \texttt{CoraML}, \texttt{CiteSeer}, and \texttt{PubMed} with the two-layer GCN of \autoref{sec:experiments}, and \texttt{CIFAR-10}, and \texttt{ImageNet-100} (its validation images) with the two-layer MLP over the $512$-d embedding of the frozen ResNet-18 (\autoref{tab:backbones}). The refits, the measures, and the two metrics are as in \autoref{sec:experiments}. Every measure is timed per evaluation point with CUDA synchronization. The fit of the base model is charged to no one, since every measure needs it, and the softmax controls are a forward pass that is already made. The refits are fitted in a batch over the candidate labels, which is how full CP is run throughout the paper. One at a time, they would cost one to two orders of magnitude more. QUAM is timed at two budgets: every class attacked for $100$ steps, and the two most probable classes for $50$. 

\autoref{fig:cost-quality} shows the result: one column per dataset. The horizontal axis is the wall-clock seconds per evaluation point (log). The vertical axis is the misclassification AUROC [Top], and the selective E-AURC [Bottom]. Markers are means over the calibration splits, and bars are $\pm1$ standard error on both axes. The two QUAM budgets are joined by a line, from the cheap one to the full one.

\begin{figure}[t]
    \centering
    \includegraphics{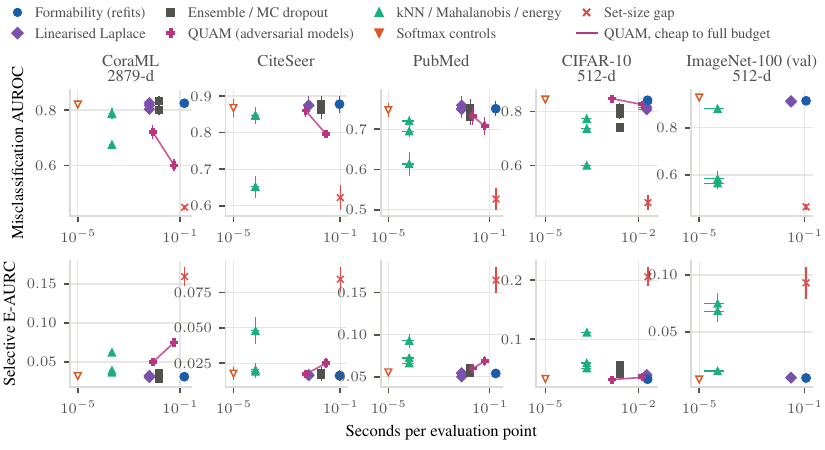}
    \caption{Cost against quality, one marker per measure, one column per dataset. [Horizontal] Wall-clock seconds per evaluation point (log), timed with CUDA synchronization. The softmax controls are a forward pass that is already made, and are drawn hollow at the left edge. [Vertical] Misclassification AUROC [Top, higher is better], and selective E-AURC [Bottom, lower is better]. Means over calibration splits, bars $\pm1$ standard error on both axes. QUAM \citep{schweighofer2023quam} is drawn at two budgets -- the two most probable classes attacked for $50$ steps, and every class for $100$ -- joined by a line from the cheap budget to the full one.}
    \label{fig:cost-quality}
\end{figure}

\subsection{Formability on Language Models}
\label{app:llm-experiments}

For language model tasks, we use four-option question answering: each model assigns probabilities to the answer letters A--D, and the conformal set contains a subset of these letters. We do not evaluate free-form generations.
Even keeping most of the parameters {frozen}, running the same pipeline as the images, and graphs can be {significantly} more expensive. Therefore we approximate formability by a local update of a LoRA adapter. In general we follow the full conformal approximation from \citet{martinez2023approximating}, but this time over the LoRA adapter parameters, {for the classification task}.

\parbold{Setup.} We use six model--dataset combinations: Llama-1B, Llama-3B{ \citep{grattafiori2024llama}}, Qwen-1.5B, and Qwen-3B \citep{qwen2025qwen25} on MMLU \citep{hendrycks2021measuring} humanities and business, and Llama-1B on MedQA \citep{jin2021disease} and MedMCQA \citep{pal2022medmcqa}. The pretrained language model is frozen, and a LoRA adapter is fitted on $2000$ labeled training questions.
The {LoRA adapter is applied} with rank $r=8$, and scaling $\alpha=16$, dropout 0, on query and value matrices of the attention layers. The same setup is also used for the ensemble model.
From separate pools of $6000$ calibration questions, and $800$ evaluation questions, we vary the calibration set size from $50$ to $2000$, with $20$ calibration draws. The score is the probability assigned to an answer letter. Every uncertainty measure ranks the same evaluation questions, using the misclassification AUROC, and selective E-AURC as mentioned in \autoref{sec:experiments}. Softmax entropy, kNN, Mahalanobis and the ensemble are not computed from $p_\vtheta$. They come from a separate cached model fitted on the training questions plus one fixed set of calibration questions. Formability is however computed from damped influence steps that start at that same anchor adapter. The comparison is not entirely fair, but it is the best we can do given the expensive computational cost of the refits on LLMs.

\parbold{Approximate refits.} Let $\tilde\vtheta_n$ denote the adapter fitted on the training questions, and updated by a single damped influence step over calibration points, and let $\tilde\vtheta_{n,c}$ denote the adapter fitted on those same questions plus $(\xtest,c)$. These are the base, and refit whose difference defines formability. The pseudo-labeled test question is the only additional observation. We approximate the two fits by damped influence steps in the adapter parameters \citep{koh2017influence,martinez2023approximating}, rather than optimizing a new adapter for every candidate. In particular, the additional step from the base has the form
\begin{equation}
    \tilde\vtheta_{n,c}-\tilde\vtheta_n\ \approx\ -\tfrac{1}{N+1}(H+\lambda\mI)^{-1}g_c,
    \label{eq:llm-one-step}
\end{equation}
where $g_c=\nabla_\vtheta\ell(\xtest,c;\vtheta_{\mathrm{ref}})$, and $\vtheta_{\mathrm{ref}}$ is the adapter at which the gradients are cached. Here $H$ is the \emph{empirical Fisher} of the existing training questions (note that a preconditioner containing the calibration questions, but not the test question would break their exchangeability). Writing $N$ for the combined size of the training, and calibration questions, and $g_i=\nabla_\vtheta\ell(\vx_i,y_i;\vtheta_{\mathrm{ref}})$, we use
\begin{equation}
    H=\frac{1}{n_{\mathrm{train}}}\sum_{i=1}^{n_{\mathrm{train}}}g_i g_i^\top.
    \label{eq:llm-empirical-fisher}
\end{equation}
Each $g_i$ uses the question's observed label $y_i$. This differs from the generalized Gauss--Newton matrix (the model Fisher for softmax cross-entropy), which averages $g_{i,c}g_{i,c}^\top$ over all four labels with weights $p(c\mid\vx_i)$.
This observed-label outer product is not, in general, a reliable approximation to the model Fisher or the Hessian \citep{kunstner2019limitations}.
Cached gradients, and their bilinear products give the changes of the candidate probability, and the calibration scores without running every refit in the main sweep. We vary the $\lambda$ {independently of} the calibration size (every calibration size is run at each $\lambda$). \autoref{fig:formability-llm} shows the results at $\lambda/\mathrm{ref}=1.5${, with $\mathrm{ref}=\mathrm{tr}(H)/n_{\mathrm{train}}$}. With $\tilde p_n$, and $\tilde p_{n,c}$ denoting the predictions of the two approximate fits, we compute $\hat w_{n,c}=\tilde p_n(c\mid\xtest)$, $\hat\Phi_{n,c}=\tilde p_{n,c}(c\mid\xtest)-\tilde p_n(c\mid\xtest)$, and $\sum_c\hat V_{n,c}=\sum_c\hat w_{n,c}\hat\Phi_{n,c}$. The same approximate refit changes the calibration scores, and hence the full-conformal threshold. This is a local approximation to adapter refitting, and does not update the frozen language-model weights.

\parbold{Result.} Similar to the image and graph experiments, the approximate formability shows a performance that is competitive with the best method and performs better than many other approaches. {However, the} performance of the method is not as strong as the softmax entropy, and in general the results are not at all comparable with the image baseline. Our intuition is that with {all the compromises we made} to make the experiments feasible (and the training to fit on a single GPU), it is very much possible that the refit approximation is not as accurate as it should be. {Finding} a more efficient way to compute the refit, or a better approximation is an interesting direction for future work. The results are shown in \autoref{fig:formability-llm}. 

\begin{figure}[H]
    \centering
    \includegraphics[width=\linewidth]{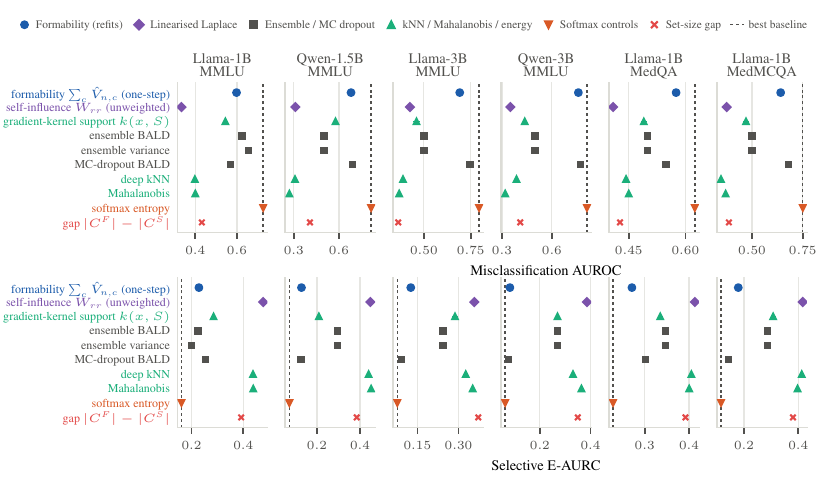}
    \caption{Formability on four-option language-model tasks. Each column ranks errors on the same evaluation questions. [Top] Misclassification AUROC (higher is better). [Bottom] Selective E-AURC (lower is better). 
    }
    \label{fig:formability-llm}
\end{figure}

\end{document}